\documentclass[final,journal,twocolumn]{IEEEtran}

\usepackage{paralist}
\usepackage[usenames,dvipsnames]{xcolor}
\usepackage{amsthm,amsmath,amsfonts,ed,empheq}
\usepackage{xspace}
\usepackage{url}
\usepackage{multirow}
\usepackage{footnote}
\usepackage{amssymb,amsmath,mathtools,fixmath}
\usepackage{algorithmicx,algorithm}
\usepackage{algpseudocode} 
\usepackage{cases}
\usepackage{cite}
\usepackage{graphicx,amsfonts,psfrag,url,booktabs}
\usepackage{comment}
\usepackage{cite}
\usepackage{xcolor}
\usepackage[caption=false,font=footnotesize]{subfig}
\usepackage{makecell}
\usepackage{psfrag}
\usepackage{tabularx}
\usepackage[table]{colortbl}
\usepackage{booktabs}
\usepackage{balance}
 
\newcommand{\x}{\mathbold{x}}

\newcommand{\y}{\mathbold{y}}

\newcommand{\z}{\mathbold{z}}
\newcommand{\s}{\mathbold{s}}
\newcommand{\e}{\mathbold{e}}

\renewcommand{\t}{\mathbold{t}}
\newcommand{\h}{\mathbold{h}}
\newcommand{\n}{\mathbold{n}}
\newcommand{\bphi}{\mathbold{\phi}}
\newcommand{\bvarphi}{\mathbold{\varphi}}
\newcommand{\bPhi}{\mathbold{\Phi}}
\newcommand{\bPsi}{\mathbold{\Psi}}
\renewcommand{\u}{\mathbold{u}}
\renewcommand{\v}{\mathbold{v}}

\newcommand{\LTV}{{\L_{\textit{tv}}}}
\newcommand{\LTVtau}{{\L_{\textit{tv}}^\tau}}
\newcommand{\bGamma}{\mathbold{\Gamma}}
\newcommand{\transp}{\mathsf{T}}
\newcommand{\Lap}{\mathbold{L}}

\renewcommand{\P}{\mathbold{P}}
\newcommand{\I}{\mathbold{I}}
\newcommand{\bSigma}{\mathbold{\Sigma}}
\renewcommand{\S}{\mathbold{S}}
\newcommand{\M}{\mathbold{M}}

\newcommand{\A}{\mathbold{A}}

\newcommand{\B}{\mathbold{B}}
\newcommand{\C}{\mathbold{C}}
\newcommand{\U}{\mathbold{U}}

\newcommand{\D}{\mathbold{D}}
\renewcommand{\H}{\mathbold{H}}

\newcommand{\eg}{e.g.,\/~}
\newcommand{\ie}{i.e.,\/~}

\newcommand{\cf}{cf.\/~}
\newcommand{\wrt}{w.r.t.\xspace}                     
\newcommand{\ARMA}[1]{\textrm{ARMA}$_{#1}$\xspace}                    
\newcommand{\FIR}[1]{\textrm{FIR}$_{#1}$\xspace}

\renewcommand{\paragraph}[1]{\vspace{2mm}\noindent \textbf{#1}}
\renewcommand{\(}{\left(}
\renewcommand{\)}{\right)}

\DeclarePairedDelimiter{\norm}{\lVert}{\rVert}
\DeclarePairedDelimiter{\abs}{\lvert}{\rvert}
\DeclareMathOperator*{\argmin}{argmin}

\newcommand{\Graph}{\mathcal{G}}

\renewcommand{\L}{\Lap}
\newcommand{\lmin}{\lambda_{\textit{min}}}

\newcommand{\lmax}{\lambda_{\textit{max}}}

\newcommand{\mrad}{{\varrho}}

\newtheoremstyle{slplain}
  {0.7\baselineskip\@plus.2\baselineskip\@minus.2\baselineskip} 
  {0.0\baselineskip\@plus.2\baselineskip\@minus.2\baselineskip}
  {\slshape}{}{\itshape}{.}{ }{}
\theoremstyle{slplain}
\newtheorem{theorem}{Theorem}

\newtheorem{remark}{Remark}

\newcommand{\annotate}[1]{{\color{black}#1}}
\newcommand{\rev}[1]{{#1}\xspace} 
\newcommand{\revv}[1]{{#1}\xspace}
\newcommand{\revvv}[1]{{#1}\xspace}  

\providecommand{\keywords}[1]{\textbf{\textit{Keywords---}} #1}

\makesavenoteenv{tabular}

\title{Distributed\! Time-Varying\! Graph\! Filtering\!}
\title{Autoregressive Moving\! Average\! Graph\! Filtering\!}

\begin{document}

\author{\IEEEauthorblockN{Elvin~Isufi$^{*}$,~\IEEEmembership{Student Member,~IEEE,}
        Andreas Loukas$^{*}$,~\IEEEmembership{Member,~ACM,}
         Andrea Simonetto,~\IEEEmembership{Member,~IEEE,}
        and~Geert Leus,~\IEEEmembership{Fellow,~IEEE}}
\thanks{$^{*}$Authors contributed equally in the preparation of this manuscript.}
\thanks{E. Isufi, A. Simonetto and G. Leus are with the faculty of Electrical Engineering, Mathematics and Computer Science, Delft University of Technology, 2826 CD Delft, The Netherlands. A. Loukas is with Department of Telecommunication Systems, TU Berlin, Germany. E-mails: {$\{$e.isufi-1, a.simonetto, g.j.t.leus$\}$@tudelft.nl}, a.loukas@tu-berlin.de. This manuscript presents a generalisation of~\cite{Loukas2015c}. This research was supported in part by STW under the D2S2 project from the ASSYS program (project 10561).}
}

\maketitle

\begin{abstract}
One of the cornerstones of the field of signal processing on graphs are graph filters, direct analogues of classical filters, but intended for signals defined on graphs.
This work brings forth new insights on the distributed graph filtering problem. 
We design a family of autoregressive moving average (ARMA) recursions, which (i) are able to \emph{approximate} any desired graph frequency response, and (ii) give \emph{exact} solutions for specific graph signal denoising and interpolation problems.
The philosophy, to design the ARMA coefficients independently from the underlying graph, renders the ARMA graph filters suitable in static and, particularly, time-varying settings. The latter occur when the graph signal and/or graph topology are changing over time.
We show that in case of a time-varying graph signal our approach extends naturally to a two-dimensional filter, operating concurrently in the graph \emph{and} regular time domain. {We also derive the graph filter behavior,} as well as sufficient conditions for filter stability when the graph \emph{and} signal are time-varying.
\rev{The analytical and numerical results presented in this paper illustrate that ARMA graph filters are practically appealing for static and time-varying settings, as predicted by theoretical derivations.}
\end{abstract}


\section{Introduction}
\label{sec:intro}

Due to their ability to capture the complex relationships present in many high-dimensional datasets, graphs have emerged as a favorite tool for data analysis. Indeed, in recent years we have seen significant efforts to extend classical signal processing methods to the graph setting, where, instead of regular low-dimensional signals (\eg a temporal or spatial signals), one is interested in graph signals, \ie signals defined over the nodes of a graph~\cite{Shuman2013}. The introduction of a Fourier-like transform for graph signals brought the tool to analyze these signals not only in the node domain, but also in the graph frequency domain~\cite{Shuman2013,Sandryhaila2013,Rabbat2014}.
\revv{One of the key tools of graph signal analysis are \emph{graph filters}.} In a direct analogy to classical filters, graph filters process a graph signal by selectively amplifying its graph Fourier coefficients. This renders them ideal for a wide range of tasks, ranging from graph signal smoothing and denoising~\cite{Zhang2008,Shuman2011}, classification~\cite{Smola2003,Zhu2003,Belkin2004} and interpolation~\cite{Narang2013}, segmentation~\cite{Loukas2014}, wavelet construction~\cite{Hammond2011}, and dictionary learning~\cite{Dong2014}---among others.

Distributed implementations of filters on graphs emerged as a way of increasing the scalability of computation~\cite{Shuman2011,Sandryhaila2014,Safavi2014,Segarra2015}. \rev{In this way, a desired graph filtering operation is performed by only local information exchange between neighbors and there is no need for a node to have access to all the data.} Nevertheless, being inspired by finite impulse response (FIR) graph filters, \annotate{these methods are sensitive to time variations, such as time-varying signals and/or graphs. An alternative approach, namely distributed infinite impulse response (IIR) graph filtering, was recently proposed~\cite{Loukas2015,Shi2015}.}
Compared to FIR graph filters, IIR filters allow for the computation of a larger family of responses, and give exact rather than approximate solutions to specific denoising~\cite{Zhang2008} and interpolation~\cite{Narang2013} problems.
\annotate{Yet the issue of time variations has so far been unresolved.} 

In a different context, we introduced IIR filter design (in fact, prior to \cite{Shi2015}) using an autoregressive process called the potential kernel~\cite{Loukas2013,Loukas2014}. These graph filters were shown to facilitate information processing tasks in sensor networks, such as smoothing and event region detection, but, due to their ad-hoc design, they only accomplished a limited subset of filtering objectives.
In this paper, we build upon our prior work to develop more general autoregressive moving average (ARMA) graph filters of any order, using parallel or periodic concatenations of the potential kernel. The design philosophy of these graph filters allows for the approximation of any desired graph frequency response {without} knowing the structure of the underlying graph. In this way, we design the filter coefficients independently of the particular graph. This allows the ARMA filters to be \emph{universally} applicable for any graph structure, and in particular when the graph varies over time, or when the graph structure is unknown to the designer.

Though ARMA graph filters belong to the class of IIR graph filters, they have a distinct design philosophy which bestows them the ability to filter graph signals not only in the graph frequency domain, but also in the regular temporal frequency domain (in case the graph signal is time-varying). 
Specifically, our design extends {naturally} to time-varying signals leading to two-dimensional ARMA filters: a filter in the graph domain as well as a filter in the time domain. 

\vspace{2mm}\noindent Our contributions are twofold: 

\vspace{2mm}\noindent\emph{(i) Distributed graph filters (Sections~\ref{sec:arma} and~\ref{sec:design})}. We propose two types of autoregressive moving average (ARMA) recursions, namely the parallel and periodic implementation, which attain a rational graph frequency response. Both methods are implemented distributedly, attain fast convergence, and have message and memory requirements that are linear in the number of graph edges and the approximation order.
Using a variant of Shanks' method, we are able to design graph filters that approximate any desired graph frequency response. \rev{In addition, we give exact closed-form solutions for tasks such as Tikhonov and Wiener-based graph signal denoising and graph signal interpolation under smoothness assumptions.}

\vspace{2mm}\noindent\emph{(ii) Time-varying graphs and signals (Section~\ref{sec:time_variations})}. We begin by providing a complete temporal characterization of ARMA graph filters \wrt time-varying graph signals.
Our results show that the proposed recursions naturally extend to two-dimensional filters operating simultaneously in the graph-frequency domain and in the time-frequency domain. 
We also discuss the ARMA recursion behavior when both the graph topology and graph signal are time-varying. Specifically, we provide sufficient conditions for filter stability, and show that a decomposition basis exists (uniquely determined by the sequence of graph realizations), over which the filters achieve the same frequency response as in the static case.

Our results are validated by simulations in Section~\ref{sec:numerical_results}, and conclusions are drawn in Section~\ref{sec:conclusions}.

\paragraph{Notation and terminology.}  We indicate a scalar valued variable by normal letters (\ie $a$ or $A$); a bold lowercase letter $\mathbold{a}$ will indicate a vector variable and a bold upper case letter $\mathbold{A}$ a matrix variable. With $a_i$ and $A_{ij}$ we will indicate the eneries of $\mathbold{a}$ and $\A$, respectively. For clarity, if needed we will refer to these entries also as $[\mathbold{a}]_i$ and $[\mathbold{A}]_{i,j}$ and to the $i$-th column of $\A$ as $[\mathbold{A}]_{i}$.
We indicate by $|a|$ the absolute value of $a$ and by $\norm{\mathbold{a}}$ and $\norm{\A}$ the 2-norm and the spectral norm of the vector $\mathbold{a}$ and matrix $\A$, respectively. 
To characterize convergence, we adopt the term \emph{linear convergence}, which asserts that a recursion converges to its stationary value exponentially with time (\ie linearly in a logarithmic scale) \cite{Boyd2004}.


\section{Preliminaries}
\label{sec:preliminaries}

Consider \rev{an undirected} graph $\Graph = (\mathcal{V}, \mathcal{E})$ of $N$ nodes and $M$ edges, where $\mathcal{V}$ indicates the set of nodes and $\mathcal{E}$ the set of edges. Let $\x$ be the graph signal defined on the graph nodes, whose $i$-th component $x_i \in \mathbb{R}$ represents the value of the signal at the $i$-th node, denoted as $u_i \in V$.

\vskip-1mm\paragraph{Graph Fourier transform (GFT).} 
The GFT transforms a graph signal $\x$ into the graph frequency domain $\hat{\x}$ by projecting it into the basis spanned by the eigenvectors of the graph Laplacian $\Lap$, typically defined as the discrete Laplacian $\L_\text{d}$ or the normalized Laplacian $\L_\text{n}$. Since the Laplacian matrix of an undirected graph is symmetric, its eigenvectors $\{\bphi_n\}_{n = 1}^N$ form an orthonormal basis, and the forward and inverse GFTs of $\x$ and $\hat{\x}$ are $\hat{\x} = \bPhi^\transp\x$ and $\x =  \bPhi\hat{\x}$, respectively, 
where the $n$-th column of $\bPhi$ is indicated as $\bphi_n$.The corresponding eigenvalues are denoted as $\{\lambda_n\}_{n = 1}^N$ and will indicate the graph frequencies.
For an extensive review of the properties of the GFT, we refer to~\cite{Shuman2013,Sandryhaila2013}.
To avoid any restrictions on the generality of our approach, in the following, we present our results for a \emph{general representation matrix} $\Lap$.
We only require that $\Lap$ is \emph{symmetric} and \emph{local}: for all $i \neq j$, $L_{ij} = 0$ whenever $u_i$ and $u_j$ are not neighbors and $L_{ij} = L_{ji}$ otherwise. We derive our results for a class of graphs with general Laplacian matrices in some restricted set $\mathcal{L}$. We assume that for every $\L \in \mathcal{L}$ the minimum eigenvalue is bounded below by $\lmin$ and the maximum eigenvalue is bounded above by $\lmax$. Hence, all considered graphs have a bounded spectral norm, i.e., $\|\L\| \leq \varrho = \max\{ |\lmax|, |\lmin|\}$. For instance, when $\L = \L_\text{d}$, we can take $\lmin=0$ and $\lmax = l$, with $l$ related to the maximum degree of any of the graphs. When ${\L} = {\L}_\text{n}$, we can consider $\lmin=0$ and $\lmax = 2$.

\vskip-1mm\paragraph{Graph filters.} A \emph{graph filter} $\H$ is an operator that acts upon a graph signal $\x$ by amplifying or attenuating its graph Fourier coefficients as 
\vspace{-1mm}
\begin{align}
	\H \x = \sum\limits_{n = 1}^N H(\lambda_n) \, \langle \x , \bphi_n \rangle  \bphi_n,
\end{align} 
where $\langle \cdot \rangle$ denotes the usual inner product. Let $\lmin$ and $\lmax$ be the minimum and maximum eigenvalues of $\Lap$ over \emph{all} possible graphs. The graph frequency response $H : [\lmin, \, \lmax] \rightarrow \mathbb{R}$ controls how much $\H$ amplifies the signal component of each graph frequency
\vspace{-1mm}
\begin{align}
	H(\lambda_n) = \langle \H\x, \bphi_n \rangle / \langle \x, \bphi_n \rangle.
\end{align} 
\vspace{-4mm}

We are interested in how we can filter a signal with a graph filter $\H$ having a user-provided frequency response $H^\ast(\lambda)$. Note that this prescribed $H^\ast(\lambda)$ is a continuous function in the graph frequency $\lambda$ and describes the desired response for {\em any} graph. This approach brings benefits in those cases when the underlying graph structure is not known to the designer, or in cases the graph changes in time. The corresponding filter coefficients are thus independent of the graph and \textit{universally} applicable. Using universal filters, we can design a single set of coefficients that instantiate the same graph frequency response $H^\ast(\lambda)$ over different bases. To illustrate the universality property, consider the application of a universal graph filter to two different graphs $\Graph_1$ and $\Graph_2$ of $N_1$ and $N_2$ nodes with graph frequency sets $\{\lambda_{1, n}\}_{n = 1}^{N_1}$ $\{\lambda_{2, n}\}_{n = 1}^{N_2}$, and eigenvectors $\{\bphi_{1,n}\}_{n = 1}^{N_1}$, $\{\bphi_{1,n}\}_{n = 1}^{N_2}$. The filter will attain the same response $H^*(\lambda)$ over both graphs, but, in each case, supported over a different set of graph frequencies: For $\Graph_1$, filtering results in $\sum_{n = 1}^{N_1} H^\ast(\lambda_{1,n}) \, \langle \x , \bphi_{1,n} \rangle  \bphi_{1,n}$, whereas for $\Graph_2$ the filtering operator will be $\sum_{n = 1}^{N_2} H^\ast(\lambda_{2,n}) \, \langle \x , \bphi_{2,n} \rangle  \bphi_{2,n}.$ Thus, the universality lies in the correctness to implement $H^\ast(\lambda)$ on all graphs, which renders it applicable for time-varying graph topologies.
%

\vskip-1mm\paragraph{Distributed implementation.} It is well known that we can \emph{approximate} a universal graph filter $\H$ in a distributed way using a $K$-th order polynomial of $\Lap$, for instance using Chebychev polynomials~\cite{Shuman2011}. Define \FIR{K} as the $K$-th order approximation given by 
\vspace{-1mm}
\begin{align}
	\H = h_0 {\I} + \sum_{k=1}^{K} h_k \Lap^{k}, 
\end{align}
%
where the coefficients $h_i$ can be both found by Chebyshev polynomial fitting~\cite{Shuman2011} or in a least-squares sense, after a (fine) gridding of the frequency range, by minimizing
\begin{equation}
\int_{\lambda} | \sum_{k=0}^K h_k \lambda^{k} - H^\ast(\lambda)|^2 \mathrm{d} \lambda.\vspace{-1mm}
\end{equation}
Observe that, in contrast to traditional graph filters, the order of the considered {\it universal} graph filters is not necessarily limited to $N$. By increasing $K$, we can approximate any filter with square integrable frequency response arbitrarily well. On the other hand, a larger FIR order implies a longer convergence time in a distributed setting, since each node requires information from all its $K$-hop neighbors to attain the desired filter response.

To perform the filter \emph{distributedly} in the network $\Graph$, we assume that each node $u_i\in \mathcal{V}$ is imbued with memory, computation, and communication capabilities and is in charge of calculating the local filter output $[\H\x]_i$. To do so, the node has to its disposal direct access to $x_i$, as well as indirect access to the memory of its neighbors.
For simplicity of presentation, we pose an additional restriction to the computation model: we will assume that nodes operate in synchronous rounds, each one  consisting of a message exchange phase and a local computation phase. In other words, in each round $u_i$ may compute any (polynomial-time computable) function which has as input, variables from its local memory as well as those from the memory of its neighbors in $\Graph$. Since the algorithms examined in this paper are, in effect, dynamical systems, in the following we will adopt the term \emph{iteration} as being synonymous to rounds. Furthermore, we assume that each iteration lasts exactly one time instant. In this context, the \emph{convergence time} of an algorithm is measured in terms of the number of iterations the network needs to perform until the output closely reaches the \emph{steady-state}, \ie the asymptotic output of the dynamical system.

The computation of \FIR{K} is easily performed distributedly. Since $\Lap^K \x = \Lap \left(\Lap^{K-1} \x \right) $, each node $u_i$ can compute the $K$-th term from the values of the $(K-1)$-th term in its neighborhood, in an iterative manner. 
The algorithm performing the FIR$_K$ graph filter terminates after $K$ iterations, and if a more efficient recursive implementation is used~\cite{Shuman2011}, in total, each node $u_i$ exchanges $K \deg{u_i}$ values with its neighbors, meaning that, overall, the network exchanges $N K \deg{u_i}$ variables, amounting to a communication cost of $O(M K)$. Further, since for this computation each node keeps track of the values of its neighbors at every iteration, the network has a memory complexity of $O(M)$. 

However, \FIR{K} filters exhibit poor performance when the graph signal or/and graph topology are time-varying, since 
the intermediate steps of the recursion cannot be computed exactly. This is for two reasons: i) First, the distributed averaging is paused after $K$ iterations, and thus the filter output is \emph{not a steady state} of the iteration $\y_t = \Lap \y_{t-1}$, which for $t = K$ gives $\y_K = \Lap^K \x$ as above. Accumulated errors in the computation alter the trajectory of the dynamical system, rendering intermediate states and the filter output erroneous. ii) Second, the input signal is only considered during the first iteration. To track a time-varying signal $\x_1, \x_2, \ldots, \x_t$, a new FIR filter should be started at each time step $t$ having $\x_t$ as input, significantly increasing the computation, message and memory complexities.

To overcome these issues and provide a more solid foundation for graph signal processing, we study \ARMA{} graph filters. 
\vspace{-2mm}


\section{ARMA Graph Filters}
\label{sec:arma}
This section contains our main algorithmic contributions.
First, Sections~\ref{subsec:ARMA1} and~\ref{subsec:ARMAK} present distributed algorithms for implementing filters with a complex rational graph frequency response
\begin{align}
	H(\lambda) = \frac{p_n(\lambda)}{p_d(\lambda)} = \frac{ \sum_{k=0}^{K} b_k \lambda^k }{1 + \sum_{k=1}^{K} a_k \lambda^k},
	\label{eq:rational}
\end{align}  
where $p_n(\lambda)$ and $p_d(\lambda)$ are the complex numerator and denominator polynomials of order $K$. \rev{Note that this structure resembles the frequency response of temporal ARMA filters \cite{Hayes2009}, in which case $\lambda = e^{j\omega}$, with $\omega$ being the temporal frequency.}  Though both polynomials are presented here to be of the same order, this is not a limitation: different orders for $p_n(\lambda)$ and $p_d(\lambda)$ are achieved trivially by setting specific constants $a_k$ or $b_k$ to zero. 

%
\subsection{\ARMA{1} graph filters}
\label{subsec:ARMA1}

Before describing the full fledged \ARMA{K} filters, it helps first to consider a 1-st order graph filter. Besides being simpler in its construction, an \ARMA{1} lends itself as the basic building block for creating filters with a rational frequency response of any order (\cf Section~\ref{subsec:ARMAK}).  
We obtain \ARMA{1} filters as an extension of the potential kernel \cite{Loukas2013}. Consider the following 1-st order recursion
\begin{subequations}
\label{eq:ARMA1}
\begin{empheq}{align}
  & \y_{t+1} = \psi\L\y_t + \varphi\x \label{eq:ARMA1_1} \\
  & \z_{t+1} = \y_{t+1} +  c \x,	\label{eq:ARMA1_2}
\end{empheq}
\end{subequations}
for arbitrary $\y_{0}$ and 
where the coefficients $\psi$, $\varphi$ and $c$ are complex numbers (to be specified later on). For this recursion, we can prove our first result.
\begin{theorem}
\label{theorem:ARMA1}
The frequency response of the \ARMA{1} is 
\begin{align}
	H(\lambda) = c + \frac{r}{\lambda - p}, \quad \text{subject to} \quad |p| > \mrad
	\label{eq:ARMA1_response}
\end{align}
with residue $r$ and pole $p$ given by $r = -\varphi/\psi$ and $p = 1/\psi$, respectively, \rev{and with $\mrad$ being the spectral radius bound of $\L$.} Recursion~\eqref{eq:ARMA1} converges to it linearly, irrespective of the initial condition $\y_0$ and graph Laplacian $\L$.
\end{theorem}
\begin{proof}
We can write the recursion~\eqref{eq:ARMA1_1} at time $t$ in the expanded form as 
\vspace{-2mm}
\begin{equation}
 \label{ARMA1Ext}
 \y_{t} = (\psi\L)^{t}\y_0 + \varphi\sum_{\tau = 0}^{t-1}(\psi\L)^\tau \x.
\end{equation}
When $\abs*{\psi \mrad} < 1$ and as $t\rightarrow \infty$ this recursion approaches the steady state
\vspace{-2mm}
\begin{equation}
 \label{ARMA1Inf}
 \y = \lim_{t\to\infty}\y_{t} = \varphi\sum_{\tau = 0}^{\infty}(\psi\L)^\tau \x = \varphi\left(\I - \psi\L\right)^{-1}\x,
\end{equation}
irrespective of $\y_0$. Based on Sylvester's matrix theorem, the matrix $\I - \psi\L$ has the same eigenvectors as $\L$ and its eigenvalues are equal to $1 - \psi\lambda_n$. It is also well known that invertible matrices have the same eigenvectors as their inverse and eigenvalues that are the inverse of the eigenvalues of their inverse. Thus, 
\vspace{-2mm}
\begin{equation}
 \label{SpecExpARMA1}
 \z = \lim_{t\to\infty}\z_{t} = \y + c\x =\sum_{n = 1}^{N}\left( c + \frac{\varphi}{1-\psi\lambda_n} \right) \hat{x}_n \bphi_n ,
\end{equation}
and the desired frequency response~\eqref{eq:ARMA1_response} follows by simple algebra. We arrive at \eqref{SpecExpARMA1} by considering a specific realization of $\L$, thus the set of eigenvalues $\lambda_n \in [\lmin, \lmax]$ is discrete. However, the same result is achieved for every other graph realization matrix $\L$ with a potentially different set of eigenvalues, still in $[\lmin, \lmax]$. Thus, we can write \eqref{eq:ARMA1_response} for all $\lambda \in [\lmin, \lmax]$.
\end{proof}
\vspace{-1mm}

The stability constraint in~\eqref{eq:ARMA1_response} can be also understood from a dynamical system perspective. Comparing recursion~\eqref{eq:ARMA1} to a discrete-time state-space equation, it becomes apparent that, when the condition $|\psi\varrho| < 1$ holds, recursion~\eqref{eq:ARMA1} achieves frequency response \eqref{eq:ARMA1_response}. It is useful to observe that, since $|p| > \mrad $, an increment of the stability region can be attained, if we work with a shifted version of the Laplacian $\Lap$ and thereby decrease the spectral radius bound $\mrad$. For instance, the following shifted Laplacians can be considered: $\L = \L_\text{d} - l/2\I$ with $\lmin = -l/2$ and $\lmax = l/2$ or $\L = \L_\text{n} - \I$ with $\lmin = -1$ and $\lmax = 1$. Due to its benefits, we will use the shifted versions of the Laplacians in the filter design phase and numerical results.\footnote{\rev{Note that from Sylvester's matrix theorem, the shifted version of the Laplacians share the same eigenvectors as the original ones.}}

\rev{Recursion~\eqref{eq:ARMA1} leads to a simple distributed implementation of a graph filter with 1-st order rational frequency response: at each iteration $t$, each node $u_i$ updates its value $y_{t,i}$ based on its local signal $x_i$ and a weighted combination of the values $y_{t-1, j}$ of its neighbors $u_j$.} Since each node must exchange its value with each of its neighbors, the message/memory complexity at each iteration is of order $O\left(M\right)$, \rev{which leads to an efficient implementation.}

\revv{
\begin{remark}\label{rem.equiv} Note that there is an \emph{equivalence} between the ARMA$_1$ filter and FIR filters in approximating rational frequency responses. Indeed, the result obtained in \eqref{ARMA1Inf} from the ARMA in $t \to \infty$ can also be achieved with an FIR filter of order $K = \infty$. Further, from \eqref{ARMA1Ext} we can see that in finite time, \ie $t = T$ and $\y_0 = 0$ the ARMA$_1$ output is equivalent to an FIR$_{T-1}$ filter with coefficients $[\varphi, \varphi\psi, \ldots, \varphi\psi^{T-1}]$.
\end{remark} 
This suggests that: $(i)$ the same output of an FIR filter can be obtained form \eqref{eq:ARMA1} and (ii) the ARMA$_1$ graph filter can be used to design the FIR coefficients to approximate frequency responses of the form \eqref{eq:ARMA1_response}. As we will see later on, due to their implementation form \eqref{eq:ARMA1}, the ARMA filters are more robust than FIRs in a time-varying scenario (time-varying graph and/or time-varying graph signal).}

 \vspace{-3mm}
\subsection{\ARMA{K} graph filters}
\label{subsec:ARMAK}

Next, we use \ARMA{1} as a building block to derive distributed graph filters with a more complex frequency response. We present two constructions: The first uses a \emph{parallel} bank of $K$ \ARMA{1} filters, attaining linear convergence with a communication and memory cost per iteration of $O(K M)$. \rev{The second uses \emph{periodic} coefficients in order to reduce the communication costs to $O(M)$, \revv{while preserving the linear convergence as the parallel ARMA$_K$ filters.}}

\paragraph{Parallel \ARMA{K} filters.} A larger variety of filter responses can be obtained by adding the outputs of a parallel \ARMA{1} filter bank. Let's denote with superscript $(k)$ the terms that correspond to the $k$-th \ARMA{1} filter for $k = 1, 2, \ldots, K$. With this notation in place, the output $\z_{t}$ of the \ARMA{K} filter at time instant $t$ is
\begin{subequations}
\label{eq:ARMAK}
\begin{empheq}{align}
  & \y^{(k)}_{t+1} = \psi^{(k)} \L \y^{(k)}_t + \varphi^{(k)}\x \label{eq:ARMAK_1} \\
  & \z_{t+1} = \sum_{k = 1}^{K} \y^{(k)}_{t+1} +  c \x,	\label{eq:ARMAK_2}
\end{empheq}
\end{subequations}

where $\y^{(k)}_{0}$ is arbitrary.
\begin{theorem}
\label{cor:ARMAK_parallel}
The frequency response of a parallel \ARMA{K} is
\begin{equation}\label{eq:ARMA_Kpartial}
	H\left(\lambda\right) = c + \sum_{k = 1}^{K}\frac{r_k}{\lambda - p_k}\quad \text{subject to} \quad	|p_k| > \mrad,
\end{equation}
with the residues $r_k = -\varphi^{(k)}/\psi^{(k)}$ and poles $p_k = 1/\psi^{(k)}$, and $\mrad$ the spectral radius of $\Lap$. Recursion~\eqref{eq:ARMAK} converges to it linearly, irrespective of the initial conditions $\y_0^{(k)}$ and graph Laplacian $\L$.
\end{theorem}
\begin{proof}
Follows straightforwardly from Theorem~\ref{theorem:ARMA1}.
\end{proof}
The frequency response of a parallel \ARMA{K} is therefore a rational function with numerator and denominator polynomials of order $K$ (presented here in a partial fraction form). 
In addition, since we are simply running $K$ \ARMA{1} filters in parallel, the communication and memory complexities are $K$ times that of the \ARMA{1} graph filter. \revv{Note also that the same considerations of Remark~\ref{rem.equiv} can be extended to the parallel ARMA$_K$ filter.}

\paragraph{Periodic \ARMA{K} filters.} We can decrease the memory requirements of the parallel implementation by letting the filter coefficients periodically vary in time.
Our periodic filters take the following form
\vspace{-1mm}
\begin{subequations}
\label{eq:ARMAK_periodic}
\begin{empheq}{align}
  & \y_{t+1} = (\theta_t {\bf I} + \psi_t \L) \y_{t} + \varphi_t \x \label{eq:ARMAK_periodic_1} \\
  & \z_{t+1} = \y_{t+1} +  c \x,	\label{eq:ARMAK_periodic_2}
\end{empheq}
\end{subequations}
where $\y_0$ is the arbitrary, the output $\z_{t+1}$ is valid every $K$ iterations, and coefficients $\theta_t, \psi_t, \varphi_t$ are periodic with period $K$:	$\theta_t = \theta_{t-iK}, \psi_t = \psi_{t-iK}, \varphi_t = \varphi_{t-iK}$, with $i$ an integer in $[0, t/K]$. 
\vspace{-1mm}

\begin{theorem}\label{theo:periodic_static}
The frequency response of a periodic \ARMA{K} filter is 
\vspace{-1mm}
\begin{align}
	H(\lambda) &= c + \frac{ \sum\limits_{k = 0}^{K-1} \left(\prod\limits_{\tau = k+1}^{K-1} \hspace{-1mm}\left(\theta_{\tau} + \psi_{\tau}\lambda\right)  \right) \varphi_{k}}{1 - \prod\limits_{k = 0}^{K-1} \left(\theta_{k} + \psi_{k}\lambda\right)},
	\label{eq:periodic_response}
\end{align}
\rev{subject to the stability constraint 
\begin{equation}
\label{eq.stabl_period}
\left| \prod_{k = 0}^{K-1} \left(\theta_{k} + \psi_{k} \mrad\right) \right| <  1
\end{equation}
with $\mrad$ being the spectral radius bound of $\Lap$.
}
Recursion~\eqref{eq:ARMAK_periodic} converges to it linearly, irrespective of the initial condition $\y_0$ and graph Laplacian $\Lap$.
\label{theorem:ARMAK_periodic}
\end{theorem}

\vspace{1mm}
(The proof is deferred to the appendix.)

\vspace{2mm} By some algebraic manipulation, we can see that the frequency response of periodic \ARMA{K} filters is also a rational function of order $K$.
\rev{We can also observe that the stability criterion of parallel \ARMA{K} is more involved than that of the parallel implementation. As now we are dealing with $K$ \ARMA{1} graph filters interleaved in time, to guarantee their joint stability one does not necessarily have to examine them independently (requiring for instance that, for each $k$, $\abs{\theta_k + \psi_k \varrho} < 1)$. Instead, it is sufficient that the product $\left| \prod_{k = 0}^{K-1} \left(\theta_{k} + \psi_{k} \mrad\right) \right|$ is smaller than one. To illustrate this, notice that if $\theta_{k} = 0$, the periodic \ARMA{K} can be stable even if some of the \ARMA{1} graph filters it is composed of are unstable. }

{When computing a periodic \ARMA{K} distributedly, in each iteration each node $u_i$ stores and exchanges }$\deg(u_i)$ values with its neighbors, yielding a memory complexity of $O(M)$, rather than the $O(K M)$ of the parallel one (after each iteration, the values are overwritten). On the other hand, since the output of the periodic \ARMA{K} is only valid after $K$ iterations, the communication complexity is again $O(KM)$. The low memory requirements of the periodic \ARMA{K} render it suitable for resource constrained devices.


\section{ARMA Filter Design}
\label{sec:design}

In this section we focus on selecting the coefficients of our filters. We begin by showing how to approximate any given frequency response with an \ARMA{} filter, using a variant of Shanks' method~\cite{Shanks1967}. This approach gives us stable filters, ensuring the same selectivity as the universal FIR graph filters. Section~\ref{subsec:gra_den_interp} then provides explicit (and exact) filter constructions for two graph signal processing problems which were up-\rev{to-now only approximated: \emph{Tiknohov and Wiener-based signal denoising} and \emph{interpolation under smoothness assumptions}~\cite{Shuman2011, Chen2014} and~\cite{Narang2013}.}
\vspace{-3mm}

\subsection{The Filter Design Problem}
\label{subsec:design}

Given a graph frequency response $H^\ast:$ $[\lmin,\, \lmax ]$ $\rightarrow \mathbb{R}$ and filter order $K$, our objective is to find the complex polynomials $p_n(\lambda)$ and $p_d(\lambda)$ of order $K$ that minimize
\begin{align}
	\int_{\lambda}  \abs*{ {\frac{p_n(\lambda)}{p_d(\lambda)} - H^\ast\hspace{-0.5mm}(\lambda)}}^2 \textnormal{d}\lambda &= &\hspace{-3mm} \int_{{\lambda}} \hspace{-0.5mm}\Bigg|\frac{ \sum\limits_{k=0}^{K} b_k \lambda^k }{1 + \sum\limits_{k=1}^{K} a_k \lambda^k} -H^\ast\hspace{-0.5mm}(\lambda) \Bigg|^2 \textnormal{d}\lambda
\end{align}
while ensuring that the chosen coefficients result in a stable system (see constraints in Theorems~\ref{cor:ARMAK_parallel} and~\ref{theorem:ARMAK_periodic}). From $p_n(\lambda)/p_d(\lambda)$ one computes the filter coefficients $(\psi^{(k)}, \varphi^{(k)}, c$ or $\theta_t, \psi_t, c)$ by algebraic manipulation.

\begin{remark} Even if we constrain ourselves to pass-band filters and we consider only the set of $\Lap$ for which $\mrad = 1$, it is impossible to design our coefficients based on classical design methods developed for IIR  filters (\eg Butterworth, Chebyshev). The same issue is present also using a rational fitting approach, \eg Pad\'e and Prony's method. \rev{This is due to the fact that, now, the filters are rational in the variable $\lambda$ and the notion of frequency does not stem from $j\omega$ nor from $e^{j\omega}$. This differentiates the problem from the design of continuous and discrete time filters. Further, the stability constraint of \ARMA{K} is different from classical filter design, where the poles of the transfer function must lie within (not outside) the unit circle.}
\end{remark}

\begin{figure}[t]
\centering
\includegraphics[width=.95\columnwidth]{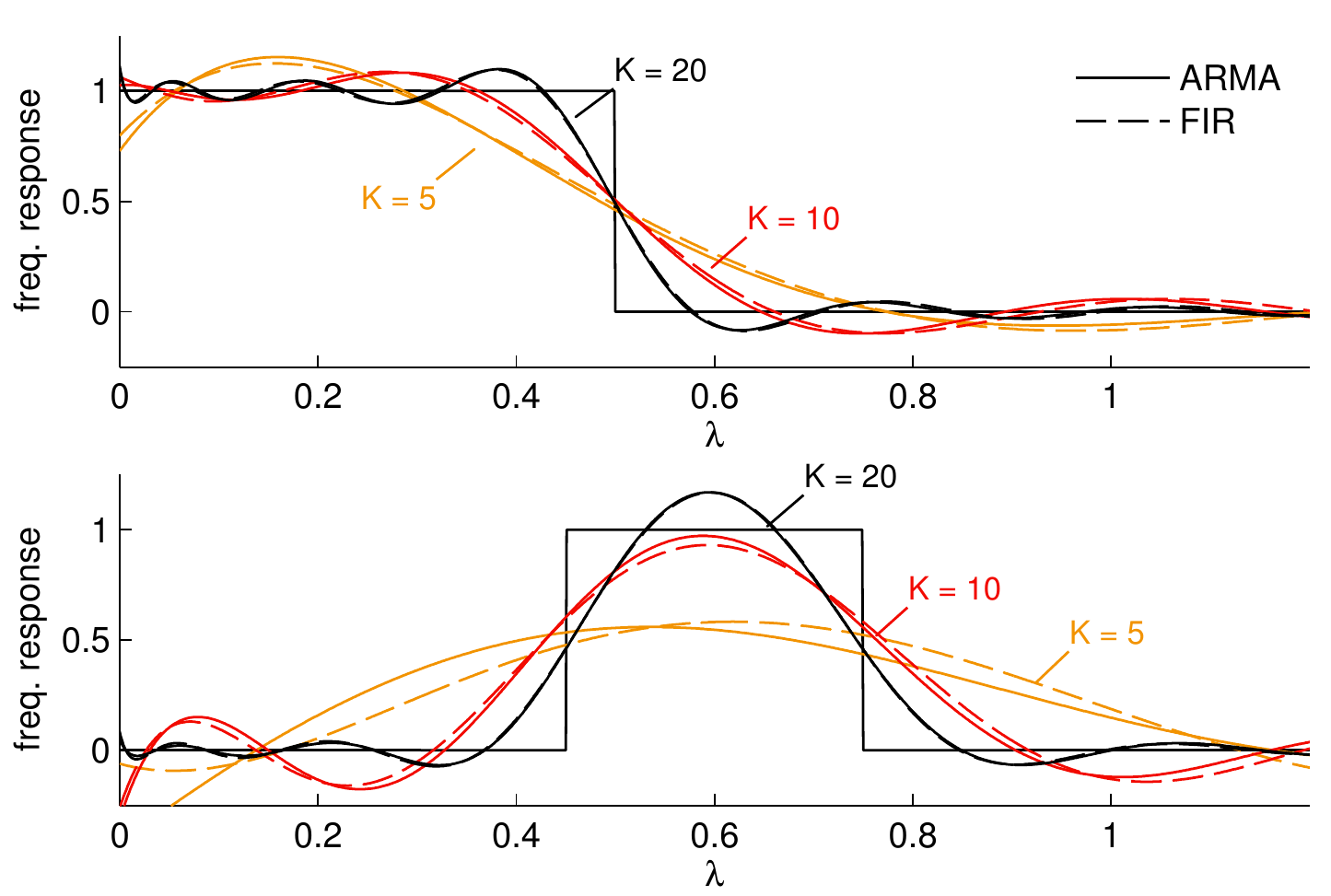}
\caption{The frequency response of \ARMA{K} filters designed by Shanks' method and the FIR responses of corresponding order. Here, $H^\ast(\lambda)$ is a step function (top) and a window function (bottom).}
\vspace{-3mm}
\label{fig:responses}
\vskip-1mm
\end{figure}

To illustrate this remark, consider the Butterworth-like graph frequency response $h(\lambda) = \left(1 + (\lambda/\lambda_{c})^K\right)^{-1}$, where $\lambda_c$ is the desired cut-off frequency. \rev{For $K = 2$, one finds that it has two complex \revv{conjugate poles at $\pm \jmath\lambda_c$.}} Thus, the behavior of these filters depends on the cut-off frequency and stability is not always guaranteed. \rev{For this particular case, and for a parallel implementation, whenever $\lambda_c > \varrho$ the filters are not stable.}%

\paragraph{Design method.} Similar to Shanks' method, we approximate the filter coefficients as follows:

\vskip2mm\emph{Denominator.} Determine $a_k$ for $k = 1, \ldots, K$ by finding a $\hat{K} > K$ order polynomial approximation $\hat{H}(\lambda) = \sum_{k=0}^{\hat{K}} g_k \lambda^k$ of $H^\ast(\lambda)$ using polynomial regression, and solving the coefficient-wise system of equations $p_d(\lambda) \hat{H}(\lambda) = p_n(\lambda)$. 

\vskip2mm\emph{Numerator.} Determine $b_k$ for $k = 1, \ldots, K$ by solving the least-squares problem of minimizing $\int_{{\lambda}}|p_n(\lambda)/p_d(\lambda) - H^\ast(\lambda)|^2 d\mu$, \wrt $p_n(\lambda)$. 

\vskip2mm
\revv{Once the numerator ($b_k$) and denominator ($a_k$) coefficients of the target rational response are found:}

\revv{\vskip2mm\emph{(i) Parallel design.} Perform the partial fraction expansion to find the residuals ($r_k$) and poles ($p_k$). Then, the filter coefficients $\psi^{(k)}$ and $\varphi^{(k)}$ can be found by exploiting their relation with $r_k$ and $p_k$ in Theorem~\ref{cor:ARMAK_parallel}.}

\revv{\vskip2mm\emph{(ii) Periodic design.} Identify $\psi_k$ by computing the roots of the (source) denominator polynomial $1 - \prod_{k = 0}^{K-1} \left(\theta_{k} + \psi_{k}\lambda\right)$ in \eqref{eq:periodic_response} and equating them to the roots of the (target) denominator $1 + \sum_{k=1}^{K} a_k \lambda^k$. It is suggested to set $\theta_1 = 0$ and $\theta_k = 1$ for $k>0$, which has the effect of putting the two polynomials in similar form. Once coefficients $\psi_k$ (and $\theta_k$) have been set, we obtain $\varphi_k$ by equating the numerator target and source polynomials. 
}

\vskip3mm 
The method is also suitable for numerator and denominator polynomials of different orders. We advocate however the use of equal orders, because it yields the highest approximation accuracy for a given communication/memory complexity.

The most crucial step is the approximation of the denominator coefficients. By fitting $p_d(\lambda)$ to $\hat{g}(\lambda)$ instead of $g(\lambda)$, we are able to compute coefficients $a_k$ independently of $b_k$. 
Increasing $\hat{K} \gg K$ often leads to a (slight) increase in accuracy, but at the price of slower convergence and higher sensitivity to numerical errors (such as those caused by packet loss). Especially for sharp functions, such as the ones shown in Fig.~\ref{fig:responses}, a high order polynomial approximation results in very large coefficients, which affect the numerically stability of the filters and push the poles closer to the unit circle. For this reason, in the remainder of this paper we set $\hat{K} = K + 1$. 


Though the proposed design method does not come with theoretical stability guarantees, it has been found to consistently produce stable filters\footnote{This has been observed while working with shifted Laplacians, and especially with the shifted normalized Laplacian $\L = \L_\text{n} - \I$.}. Fig.~\ref{fig:responses} illustrates in solid lines the frequency responses of three \ARMA{K} filters ($K=5,10,20$), designed to approximate a step function (top) and a window function (bottom). \rev{For reproducibility, Table I, which is featured in the Appendix, summarizes the filter coefficients of the parallel implementation for different $K$.} 

\begin{figure}[t]
\centering
\includegraphics[width=.95\columnwidth]{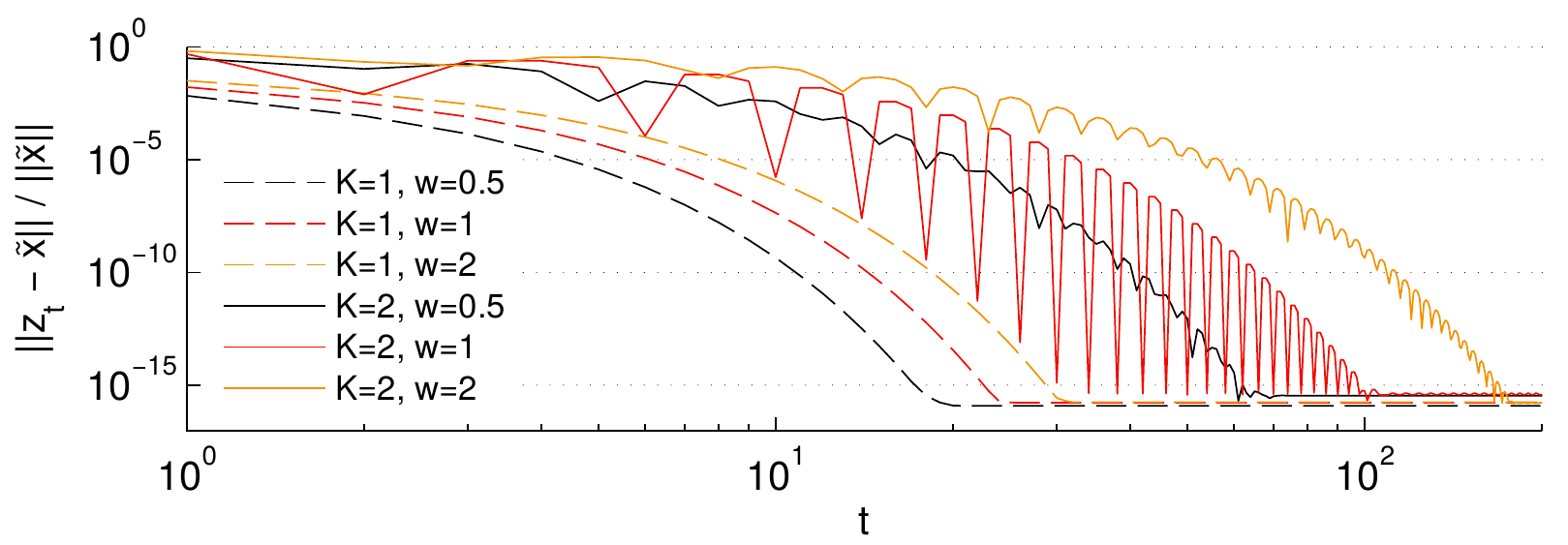}
\vskip-1mm
\caption{Convergence of a denoising parallel \ARMA{K} filter for $K = 1, 2$ and $w = 0.5, 1, 2$. The filtering error is $\norm{\z_t-\tilde{\x}} / \norm{\tilde{\x}}$, where for each parameter pair $(K,w)$, $\tilde{\x}$ is the solution of the denoising problem, and $\z_t$ is the filter output after $t$ rounds. The residual error is a consequence of the computer's bounded numerical precision.}
\vspace{-3mm}
\label{fig:comparison}
\vskip-1mm
\end{figure}

\subsection{Exact and Universal Graph Filter Constructions}
\label{subsec:gra_den_interp}

We proceed to present exact (and in certain cases explicit) graph filter constructions for particular graph signal denoising and interpolation problems. In contrast to previous work, the proposed filters are \emph{universal}, that is they are designed without knowledge of the graph structure. Indeed, the filter coefficients are found independently from the eigenvalues of the graph Laplacian. This makes the ARMA filters suitable for any graph, and ideal for cases when the graph structure is unknown or when the $O(N^3)$ complexity of the eigenvalue decomposition becomes prohibitive.

\paragraph{Tikhonov-based denoising.} Given a noisy signal $\t = \x + \n$, where $\x$ is the true signal and $\n$ is noise, the objective is to recover $\x$~\cite{Shuman2011,Shuman2013,Chen2014}. When $\x$ is smooth \wrt the graph, denoising can be formulated as the regularized problem
\begin{align}
	\tilde{\x} = \argmin_{\x \in \mathbb{C}^N} \norm{\x - \t}_2^2 + w \, \x^\top \L^K \x,
	\label{eq:denoising}
\end{align}
where the first term asks for a denoised signal that is close to $\t$, and the second uses the quadratic form of $\L^K$ to penalize signals that are not smooth. \rev{In \eqref{eq:denoising}, admitted choices of $\L$ are limited to the discrete Laplacian $\L_\text{d}$ or the normalized Laplacian $\L_\text{n}$ (without shifting).} The positive regularization weight $w$ allows us to trade-off between the two terms of the objective function. Being a convex problem, the global minimum $\tilde{\x}$ is found by setting the gradient to zero, resulting in 
\begin{align}
	\tilde{\x} &= \left(\I + w \L^K \right)^{-1} \t = \sum_{n = 1}^N \frac{1}{1 + w \lambda_n^K}  \langle\t,\bphi_n\rangle \bphi_n.
	\label{eq.solARMAK}    
\end{align}
\rev{It follows that \eqref{eq.solARMAK} can be approximated with an \ARMA{K} with frequency response
\begin{align}\label{eq:den_ARMA}
	H(\lambda) &= \frac{1}{1 + w \lambda^K } = \frac{1}{\prod_{k = 1}^K ( \lambda - p_k) } 
\end{align}
\rev{with $p_k = - e^{\jmath\gamma_k}/\sqrt[K]{w}$ and $\gamma_k = (2k + 1)\pi/K$.} From the stability condition of the parallel \ARMA{K}, we have stable filters as long as that $\abs{p_k} > \varrho$, which for the particular expression of $p_k$ becomes $\sqrt[K]{w}\mrad < 1$.} 

\rev{The solution of \eqref{eq.solARMAK} can also computed by an \ARMA{K} implemented on the shifted Laplacians, with a notable improvement on the stability of the filters. For $\L = \L_\text{d} - l/2\I$ we can reformulate  \eqref{eq:den_ARMA} as

\begin{align}\label{eq:den_ARMA_ref}
	H(\lambda) &= \frac{1}{1 + w (\lambda + \frac{l}{2})^K } = \frac{1}{\prod_{k = 1}^K ( \lambda - p_k) } 
\end{align}
\rev{where now $p_k = -{l}/{2} + e^{\jmath\gamma_k}/\sqrt[K]{w}$ for $\gamma_k = (2k + 1)\pi/K$.} Again, from the stability of \ARMA{K} $|p_k| > \mrad$, or equivalently $|p_k|^2 > \mrad^2$, we now obtain stable filters as long as
\begin{align}
\left(-\frac{l}{2} + \frac{\text{cos}(\gamma_k)}{\sqrt[K]{w}}\right)^2 + \frac{\text{sin}^2(\gamma_k)}{\sqrt[K]{w}^2} > \mrad^2,
\end{align}
or equivalently
\begin{align}
\left(\frac{l^2}{4} - \mrad^2\right)\sqrt[K]{w}^2 - l~\text{cos}(\gamma_k)\sqrt[K]{w} + 1 > 0,
\end{align}
are satisfied. For the shifted normalized Laplacian, $\varrho = 1$ and $l = 2$, the stability condition simplifies to 
\begin{align}
2\text{cos}(\gamma_k)\sqrt[K]{w} < 1,
\label{eq.norm_K12}
\end{align}
which is always met for the standard choices of $K = 1$ (quadratic regularization) and $K=2$ (total variation)\footnote{\rev{Even though $w$ is a free parameter, for K = 1, 2 the value $\text{cos}(\gamma_k)$ in \eqref{eq.norm_K12} will be either 0 or -1, due to the expression of $\gamma_k$.}}.}  For these values of $K$, and for different values of the regularized weight $w$, we show in Fig.~\ref{fig:comparison} the normalized error between the output of the ARMA$_K$ recursion and the solution of the optimization problem \eqref{eq:denoising}, as a function of time.

For both \eqref{eq:den_ARMA} and \eqref{eq:den_ARMA_ref}, the denominator coefficients $\psi^{(k)}$ of the corresponding parallel ARMA$_K$ filter can be found as $\psi^{(k)} = 1/p_k$. Meanwhile, the numerator coefficients $\varphi^{(k)}$ are found in two steps: (i) express \eqref{eq:den_ARMA}, \eqref{eq:den_ARMA_ref} in the partial form as in \eqref{eq:ARMA_Kpartial} to find the residuals $r_k$ and (ii) take $\varphi^{(k)} = -r_k\psi^{(k)}$.

\paragraph{Wiener-based denoising.} When the statistical properties of the graph signal and noise are available, it is common to opt for a Wiener filter, \ie the linear filter that minimizes the mean-squared error (MSE)
\begin{align}
	\tilde{\H} &= \argmin_{\H \in \mathbb{R}^{N \times N}} \mathbb{E}\left[{{\norm{ \H \t - \x}^2_2 }}\right] \textnormal{and}~
	\tilde{\x} = \tilde{\H}\t,
	\label{eq:Wiener_obj}
\end{align}
where as before $\t = \x + \n$ is the graph signal which has been corrupted with additive noise. It is well known that, when $\x$ and $\n$ are zero-mean with covariance $\bSigma_\x$ and $\bSigma_\n$, respectively,
the solution of \eqref{eq:Wiener_obj} is 
\begin{align}	
	   \tilde{\x} = \bSigma_\x (\bSigma_\x + \bSigma_\n)^{-1} \t.
	   \label{eq:Wiener_solution_general}
\end{align}  
given that matrix $\bSigma_\x + \bSigma_\n$ is non-singular.
The above linear system can be solved by a graph filter when matrices $\bSigma_\x $ and $\bSigma_\n$ share the eigenvectors of the Laplacian matrix $\L$. Denote by $\sigma_\x(\lambda_n) = \bphi_n^\top \bSigma_\x \bphi_n$ the eigenvalue of matrix the $\bSigma_\x$ which corresponds to the $n$-th eigenvector of $\L$, and correspondingly $\sigma_\n(\lambda_n) = \bphi_n^\top \bSigma_\n \bphi_n$. We then have that 
\begin{align}
	\tilde{\x} = \sum_{n=1}^N \frac{ \sigma_\x(\lambda_n) }{ \sigma_\x(\lambda_n) + \sigma_{\n}(\lambda_n)} \langle \t, \bphi_n \rangle \bphi_n.
	\label{eq:wienerfilter}
\end{align}
It follows that, when $\sigma_\x(\lambda)$ and $\sigma_\n(\lambda)$ are rational functions (of $\lambda$) of order $K$, the Wiener filter corresponds to an \ARMA{K} graph filter. Notice that the corresponding filters are still universal, as the \ARMA{K} coefficients depend on the rational functions $\sigma_\x(\lambda)$ and $\sigma_\n(\lambda)$, but not on the specific eigenvalues of the graph Laplacian. \rev{In a different context, similar results have been also observed in semi-supervised learning \cite{Girault2014}.}

Let us illustrate the above with an example. Suppose that $\x$ is normally distributed with covariance equal to the pseudoinverse of the Laplacian $\L^\dagger$. This is a popular and well understood model for smooth signals on graphs with strong connections to Gaussian Markov random fields~\cite{Zhang2015}. In addition, let the noise be white with variance $w$. Substituting this into \eqref{eq:wienerfilter}, we find 
\begin{align}
	\tilde{\x} &= \sum_{n=1}^N \frac{ 1 }{ 1 + w\lambda_n} \langle \t, \bphi_n \rangle \bphi_n, 
\end{align}
which is identical to the Tikhonov-based denoising for $K=1$ and corresponds to an \ARMA{1} \rev{with $\varphi = {2}/{(2 + wl)}$ and $\psi = -{2w}/{(2 + w{l})}$, which as previously shown has always stable implementation.}
Note that even though the stability is ensured for the considered case, it does not necessarily hold for every covariance matrix. Indeed, the stability of the filter must be examined in a problem-specific manner.

\paragraph{Graph signal interpolation.} Suppose that only $r$ out of the $N$ values of a signal $\x$ are known, and let $\t$ be the $N \times 1$ vector which contains the known values and zeros otherwise. Under the assumption of $\x$ being smooth \wrt $\L = \L_\text{d}$ or $\L = \L_\text{n}$, we can estimate the unknowns by the regularized problem 
\begin{align}
	\tilde{\x} = \argmin_{\x \in \mathbb{R}^N} \norm{\S \left(\x - \t\right)}_2^2 + w \, \x^\top \L^K \x,
	\label{eq:interpolation}
\end{align}
where $\S$ is the diagonal matrix with $S_{ii} = 1$ if $x_i$ is known and $S_{ii} = 0$ otherwise. Such formulations have been used widely, both in the context of graph signal processing~\cite{Narang2013,Mao2014} and earlier by the semi-supervised learning community~\cite{Belkin2004,Zhang2006}. 
Similar to~\eqref{eq:denoising}, this optimization problem is convex and its global minimum is found as 
\begin{align}
	\tilde{\x} = \left( \S + w \L^K \right)^{-1} \t.
\end{align}
Most commonly, $K = 1$, and $\tilde{\x}$ can be re-written as 
\begin{align}
	\tilde{\x} \!=\! \Big( \I - \hat{\L} \Big)^{-1}\! \t \!= \!\sum\limits_{n = 1}^N \frac{1}{1 + \hat{\lambda}_n} \langle\t,\hat{\bphi}_n\rangle \hat{\bphi}_n.
\end{align}
which is an \ARMA{1} filter designed for the Laplacian matrix $\hat{\L} = \S - \I + w \L$ with $(\hat{\lambda}_n, \hat{\bphi}_n)$
the $n$-th eigenpair of $\hat{\L}$. 
For larger values of $K$, the interpolation cannot be computed distributedly using \ARMA{} filters. That is because the corresponding 
 basis matrix $\hat{\L} = \S + w \L^K$ cannot be appropriately factorized into a series of local matrices.


\section{Time-Variations}
\label{sec:time_variations}
At this point we have characterized the filtering and convergence properties of \ARMA{} graph filters for static inputs. But do these properties hold when the graph and signal are a function of time? In the following, we characterize ARMA graph filters with respect to time-variations in the graph signal (\cf Section~\ref{subsec:joint}), as well as in the graph topology (\cf Section~\ref{subsec:TV}). 
\subsection{Joint Graph and Temporal Filters}
\label{subsec:joint}

To understand the impact of graph signal dynamics we broaden the analysis to a two-dimensional domain: the first dimension, as before, captures the graph (based on the graph Fourier transform), whereas the second dimension captures time (based on the Z-transform \cite{Hayes2009}). 
This technique allows us to provide a complete characterization of the \ARMA{} filter subject to time-variations. First, we show that the \ARMA{} filter output remains close to the correct time-varying solution (under sufficient conditions on the input), which implies that  our algorithms exhibit a certain robustness to dynamics. Further, we realize that \ARMA{} graph filters operate along both the graph and temporal dimensions. We find that a graph naturally dampens temporal frequencies in a manner that depends on its spectrum. Exploiting this finding, we extend the \ARMA{} designs presented in Section~\ref{sec:arma} so as to also allow a measure of control over the temporal frequency response of the filters.

As previously, we start our exposition with the \ARMA{1} recursion \eqref{eq:ARMA1}, but now the input graph signal $\x_t$ is time dependent (thus, indexed with the subscript $t$)
\begin{subequations}
\label{eq:ARMA1TV}
\begin{empheq}{align}
  & \y_{t+1} = \psi\L\y_t + \varphi\x_t \\
  & \z_{t+1} = \y_{t+1} +  c \x_t.	
\end{empheq}
\end{subequations}
The dimension of the above recursion can be reduced by restricting the input graph signal to lie in the subspace of an eigenvector $\bphi$ with associated eigenvalue $\mu$, \ie $\x_t = x_t \bphi$, where now $x_t$ is a scalar \revv{and similarly, we take $\y_0 = y_0\bphi$.\footnote{This is a standard way to derive the frequency response of the system.}} \rev{By orthogonality of the basis, the filter only alters the magnitude $x_t$ relative to the eigenvector $\bphi$ and not the direction of $\x_t$.} Therefore, \eqref{eq:ARMA1TV} is equivalent to
\begin{subequations}
\begin{empheq}{align}
  &  y_{t+1} = \psi \lambda y_t + \varphi x_t \\
  & z_{t+1} = y_{t+1} +  c x_t,	
\end{empheq}
\end{subequations}	
where $x_t, y_t, z_t \in \mathbb{R}$ are simply the magnitudes of the vectors $\x_t, \y_t,  \z_t \in \mathbb{C}^n$ lying in the eigenspace of $\bphi$, and we can write $\y_t = y_t \bphi$ and $\z_t = z_t \bphi$. Taking the Z-transform on both sides, we obtain the joint graph and temporal frequency transfer function 
\begin{align}
	\label{eq:2DARMA1}
	H(z,\lambda) = \frac{\varphi z^{-1}}{1 - \psi \lambda z^{-1}} + c z^{-1}.
\end{align}
\rev{It can be shown that the temporal-impulse response for each graph frequency $\lambda$ is}
%
\begin{align}
\label{eq.joind_imp_resp}
	\revv{\h_{t+1} (\lambda) = \left( \varphi (\psi \lambda)^t + c\delta[t] \right) \bphi,}
\end{align}
\revv{with $\delta[t]$ the impulse function.} From \eqref{eq.joind_imp_resp} we deduce that the region of convergence (ROC) of the filter is $ \{ \abs{z} > \abs{\psi \lambda}$, for all $ \lambda \}$ and that the filter is causal.

The joint transfer function characterizes completely the behavior of \ARMA{1} graph filters for an arbitrary yet time-invariant graph: when $z \to 1$, we return to the constant $\x$ result and stability condition of Theorem \ref{theorem:ARMA1}, while for all other $z$ we obtain the standard frequency response as well as the graph frequency one. As one can see, recursion \eqref{eq:ARMA1TV} is an \ARMA{1} filter in the graph domain as well as in the time domain. 
Observe also that the poles of $H(z,\lambda)$ obey the fixed relationship $z = \lambda \psi$. This yields an interesting insight: the temporal frequency response of the filter differs along each graph frequency $\lambda$, meaning that temporal dynamics affecting signals lying in low graph frequency eigenspaces are dampened to a smaller extent. 

As Theorems~\ref{theorem:ARMAKTV_parallel} and~\ref{theorem:ARMAKTV_periodic} below show, the results are readily generalized to higher order filters.

\begin{theorem}
\label{theorem:ARMAKTV_parallel}
The joint graph and temporal frequency transfer function of a parallel \ARMA{K} is
\begin{align}
\label{eq:ARMAK_parallel_Trans_Function}
	H(z, \lambda) = \sum_{k = 1}^{K} \frac{\varphi^{(k)} z^{-1}}{1 - \psi^{(k)} \lambda z^{-1}} + c z^{-1},
\end{align}
subject to the stability conditions of Theorem~\ref{cor:ARMAK_parallel}.
\end{theorem}
\begin{theorem}
\label{theorem:ARMAKTV_periodic}
The joint graph and temporal frequency transfer function of a periodic \ARMA{K} is
\begin{align}
	H(z, \lambda) = \frac{\sum\limits_{k = 0}^{K-1} \left(\prod\limits_{\tau = k+1}^{K-1} \theta_\tau + \psi_\tau\lambda \right) \varphi_k z^{k-K}}{1 - \left(\prod\limits_{\tau = 0}^{K-1} \theta_\tau + \psi_\tau\lambda\right) z^{-K}} + c z^{-1},
	\label{eq:ARMAKTV_periodic}
\end{align}
subject to the stability conditions of Theorem~\ref{theorem:ARMAK_periodic}.
\end{theorem}

\vskip2mm
(The proofs are deferred to the appendix.)
\vskip2mm

\begin{figure}[t!]
\centering
\begin{tabular}{c c}
\hskip-4mm
      {\includegraphics[width = .51\columnwidth, keepaspectratio]{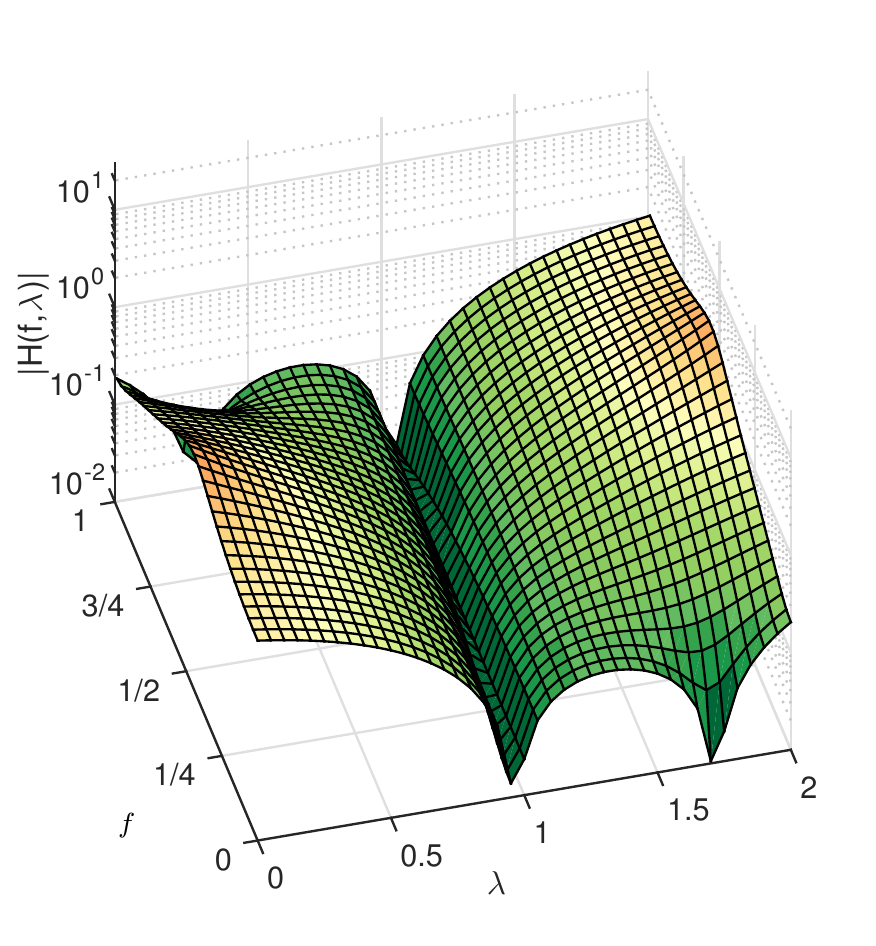}} & \hskip-3mm{\includegraphics[width = .51\columnwidth, keepaspectratio]{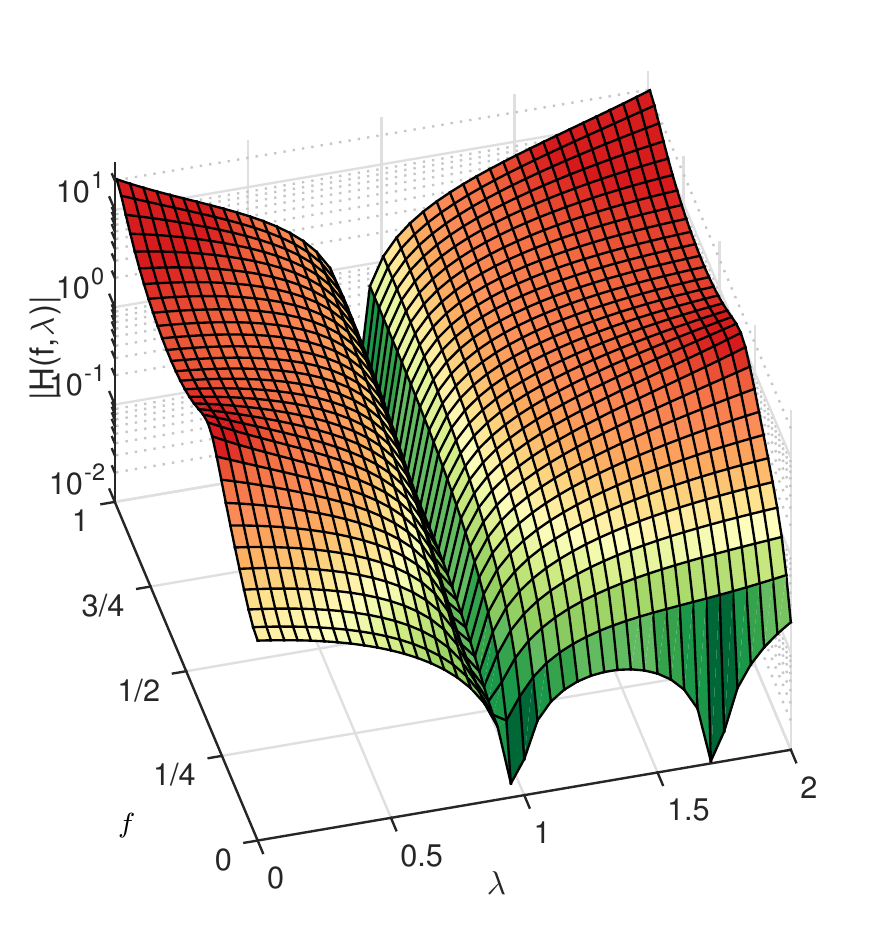}}
\end{tabular}
\caption{The joint graph and temporal frequency response of a parallel and a periodic graph filter, both designed to approximate an ideal low pass (step) \revv{response with cut-off frequency $\lambda_c = 0.5$} and $K = 3$ \wrt the normalized graph Laplacian. \rev{The temporal frequencies $f$ are normalized ($\times \pi$ rad/sample).} \label{fig:joint}}
\vspace{-5mm}
\end{figure}

As in the first order case, Theorems~\ref{theorem:ARMAKTV_parallel} and~\ref{theorem:ARMAKTV_periodic} describe completely the behavior of the parallel and periodic implementations. We can see that both filters are \ARMA{K} filters in the graph and temporal domain. 
In particular, the parallel and periodic filters have up to $K$ distinct poles abiding respectively to
\begin{align}
	z = \psi^{(k)} \lambda \quad \text{and} \quad z = \sqrt[K]{\prod_{\tau = 0}^{K-1} \theta_\tau + \psi_\tau \lambda }.
\end{align}

To provide further insight, Fig.~\ref{fig:joint} plots the joint graph and temporal frequency response of a parallel and a periodic graph filter of third order, \rev{both designed (only in the graph domain) to approximate an ideal low pass response \revv{with cut-off frequency $\lambda_\text{c} = 0.5$.} In the figure, the horizontal axis measures the graph frequency with smaller $\lambda$ corresponding to lower variation terms. The temporal axis on the other hand measures the normalized temporal frequency $f$ such that, for $f = 0$, one obtains the standard graph frequency response.}

\rev{We make two observations: \emph{First, both graph filters ensure almost the same frequency response as for the static case ($f = 0$) for low temporal variations $f \le 1/8$.} This suggests that these filters are more appropriate for slow temporal variations.}
\rev{
Whereas for graph signals lying in eigenspaces \revv{with $\lambda$ close to $\lambda = 1$ all temporal frequencies are damped.} This is a phenomenon that transcends the filter implementation (parallel or periodic) and the particular filter coefficients. It is attributed to the shifting of the Laplacian in the design phase and to the multiplicative relation of the response poles. \emph{Second, the parallel concatenation of \ARMA{1} filters results in a more stable implementation that is more fit to tolerate temporal dynamics than the periodic implementation.} As shown in Figure~\ref{fig:joint}, for $\lambda = 0$ and $\lambda=2$, temporal fluctuations with frequencies exceeding $1/8$ cause the periodic filter output to blow up by an order of magnitude, effectively rendering the periodic implementation unusable. The poor stability of the periodic implementation is also seen from~\eqref{eq:ARMAKTV_periodic}, where the $\theta_\tau$ terms tend to push the poles closer to the unit circle, and it is the price to pay for its small memory requirements. 
}
Due to its superior stability properties and convenient form (less coefficients and simpler design), we suggest the parallel implementation for dealing with time-varying graph signals.

\subsection{Time-Varying Graphs and Signals}
\label{subsec:TV}

Time variations on the graph topology bring new challenges to the graph filtering problem. \emph{First, they render approaches that rely on knowledge of the graph spectrum ineffective.} Approaches which ensure stability by designing the poles to lie outside the set of the Laplacian eigenvalues of a given graph, may lead to unstable filters in a different graph where some eigenvalues may over-shoot one of the poles. Due to their different design philosophy, the presented ARMA graph filters handle naturally the aforementioned issues. We can, for instance, think that the different graph realizations among time enjoy an upper bound on their maximum eigenvalue $\lmax$. In case this is not possible, or difficult to determine, we can always work with the normalized Laplacian and thus take $\lmax = 2$. In this way, by designing the filter coefficients in order to ensure stability \wrt $\lmax$, \rev{we automatically impose stability for all different graph realizations.} Furthermore, by designing once the filter coefficients for a continuous range of frequencies, the ARMA recursions also preserve the desired frequency response for different graph realizations.

\emph{The second major challenge is characterizing the graph filter behavior.} Time-varying affine systems are notoriously difficult to analyze when they possess no special structure~\cite{Touri2011}. To this end, we devise a new methodology for time-varying graph filter analysis. We show that a decomposition basis always exists, over which \ARMA{1} graph filters (and as a consequence parallel \ARMA{K} filters) have the same frequency response as in the static case. Furthermore, this decomposition basis depends only on the sequence of graph realizations.

In case of a time-varying graph topology, yet with a fixed number of nodes $N$, as well as a time-varying graph signal, the ARMA$_1$ recursion~\eqref{eq:ARMA1} can be written as
\begin{subequations}
\label{eq:ARMA1TV_G}
\begin{empheq}{align}
\label{{eq:ARMA1TV_G}_a}
  & \y_{t+1} = \psi\L_{t}\y_t + \varphi\x_t \\
  & \z_{t+1} = \y_{t+1} +  c \x_t,	
\end{empheq}
\end{subequations}
where the time-varying graph is shown by indexing $\L_t$ with the subscript $t$.
Expanding the recursion we find that, for any sequence of graph realizations $\{G_0, G_1, \ldots , G_{t}\}$ with corresponding Laplacians $\{\L_0, \L_1, \ldots, \L_{t}\}$, the output signal is given by
\begin{align}
\label{eq:ARMA_TVTV}
	\z_{t+1} &\!=\! \psi^{t\!+\!1} \!\bPhi_\L(t, 0) \y_{0} \!+\! \varphi\! \sum_{\tau=0}^{t} \psi^\tau \bPhi_\L(t, t-\tau+1) \x_{t-\tau} \!+\! c \x_{t}\!,
\end{align}
\rev{where $\bPhi_\L(t, t') = \L_t\L_{t-1}\ldots\L_{t'}$ for $t \ge t'$, and $\bPhi_\L(t, t') = \I$ otherwise.}

Since the output $\z_t$ depends on the entire sequence of graph realizations, the spectrum of any individual Laplacian is insufficient to derive the graph frequency of the filter. To extend the spectral analysis to the time-varying setting, we define a joint Laplacian matrix $\LTV$ that encompasses all the individual shifted graph Laplacians. \rev{The intuition behind our approach is to think of a time-varying graph as one large graph $\mathcal{G}_{\emph{tv}}$ that contains all nodes of the graphs $\mathcal{G}_0, \mathcal{G}_1, \ldots, \mathcal{G}_t$, as well as directional edges connecting the nodes at different timesteps. We then interpret the spectrum of the associated Laplacian matrix $\LTV$ as the basis for our time-varying graph Fourier transform. This idea is a generalization of the joint graph construction~\cite{Loukas2016}, used to define a Fourier transform for graph signals which change with time (though in the previous work the graph was considered time-invariant). Similar to~\cite{Loukas2016}, we will construct $\mathcal{G}_{\emph{tv}}$ by replicating each node $u_i\in V$ once for each timestep $t$. Denote the $t$-th replication of the $i$-th node as $u_{t,i}$. For each $t$ and $i$, $\mathcal{G}_{\emph{tv}}$ will then contain directed edges between $u_{t-1,j}$ and $u_{t,i}$ with $u_j$ being a neighbor of $u_i$ in $\mathcal{G}_{t-1}$. Therefore, in contrast to previous work, here the edges between nodes $u_i$ and $u_j$ are a function of time. By its construction, $\mathcal{G}_{\emph{tv}}$ captures not only the topology of the different graphs, but also the temporal relation between them: since the exchange of information between two neighbors incurs a delay of one unit of time, at each timestep $t$, a node has access to the values of its neighbors at $t-1$. }

\rev{To proceed, define $\P$ to be the $(t+1)\times (t+1)$ cyclic shift matrix with ones below the diagonal
and construct $\LTV$ as the $N(t+1) \times N(t+1)$ permuted block-diagonal matrix
\begin{align}
	\LTV = \text{blkdiag}[\L_0, \L_1, \ldots, \L_t] (\P \otimes \I),
\end{align}
\rev{For consistency with the established theory on GFT, when $t=0$ and the graph is time-invariant, we define $\P = 1$.} Let $\e_{\tau}$ be the $(t+1)$-dimensional canonical unit vector with $(e_{\tau})_i = 1 $ if $i = \tau$ and $(e_{\tau})_i = 0$, otherwise. 
Defining $\s = [\x_0^\top, \x_1^\top, \ldots, \x_{t}^\top]^\top$ as the vector of dimension $N(t+1)$ which encompasses all input graph signals, we can then write 
\begin{align}
\label{eq:oihd}
	\bPhi_\L(t,t-\tau+1) \x_{t-\tau} = (\e_{t+1}^\top\otimes\I) \, \LTVtau \, \s.
\end{align}
\revv{In those cases when the non-symmetric matrix $\LTV$ enjoys an eigendecomposition, we have $\LTV = \U \mathbold{\Lambda} \U^{-1}$ with ($\lambda_k$, $[\U]_k$) the $k$-th eigenpair. Specifically, $\lambda_k$ is the $k$-th diagonal element of $\mathbold{\Lambda}$ and $[\U]_k$ is the $k$-th column of $\U$. The total number of eigenpairs of $\LTV$ is $K = N \times (t+1)$. To ease the notation, we will denote as $[\U]^{-1}_k$ the respective $k$-th column of $\U^{-1}$.}
Substituting~\eqref{eq:oihd} into the second term of \eqref{eq:ARMA_TVTV} and rearranging the sums, we get
\begin{align}
	\hspace{0mm}\varphi \, \sum_{\tau=0}^{t} \psi^\tau \bPhi_\L(t, t-\tau+1) \x_{t-\tau} &= \varphi (\e_{t+1}^\top\otimes\I) \sum_{\tau=0}^{t} (\psi \LTV)^\tau \, \s \nonumber \\
	&\hspace{-46mm}= \varphi (\e_{t+1}^\top\otimes\I) \sum_{\tau=0}^{t} \sum_{k = 1}^{K} (\psi \lambda_k)^\tau \,  \langle\s, [\U]^{-1}_k\rangle  [\U]_k \nonumber \\
	&\hspace{-46mm}=  (\e_{t+1}^\top\otimes\I) \sum_{k = 1}^{K}  \varphi \, \frac{1 - (\psi \lambda_{k})^{t+1}}{1 - \psi \lambda_{k}} \langle\s, [\U]^{-1}_k\rangle [\U]_k.
\end{align}
Similarly, 
\begin{align}\label{eq:MTV_y0}
	\psi^{t+1} \bPhi_\L(t, 0) \y_{0} &= (\e_{t+1}^\top\otimes\I) \, \left(\psi \LTV\right)^{t+1} \, (\e_{t+1} \otimes \y_0) \nonumber \\
	&\hspace{-24mm}= (\e_{t+1}^\top\otimes\I)  \sum_{k = 1}^{K}  (\psi \lambda_{k})^{t+1} \langle\e_{t+1} \otimes \y_0, [\U]^{-1}_k\rangle [\U]_k 
\end{align}
as well as 
\begin{align}
	c \, \x_{t} = (\e_{t+1}^\top\otimes\I) \sum_{k = 1}^{K} c\,\langle \s, [\U]^{-1}_k\rangle [\U]_k.
\end{align}
\revvv{Without loss of generality, when $t$ is sufficiently large we can ignore terms of the form $(\psi \lambda_{k})^{t+1}$ as long as $|\psi \lambda_{k}| < 1$, which also indicates that the impact of the filter initialization ${\y}_0$ on the filter output vanishes with time.}This condition is met when $\|\psi\LTV\| < 1$, which as a direct consequence of Gershgorin's circle theorem, this stability condition is met if, for every $\tau$, the sum of the elements of each row of $\L_\tau$, in absolute value, is smaller than $\abs{1/\psi}$ (which also implies that the eigenvalues of $\L_\tau$ are bounded by $\abs{1/\psi}$). For the (shifted) normalized Laplacian this means that $\abs{\psi} < 2$ ($\abs{\psi} < 1$), matching exactly the stability conditions of the static case. Under this sufficient condition, the filter output approaches 
\begin{align}\label{eq:convTVTV}
	\z_{t+1} \approx (\e_{t+1}^\top\otimes\I) \sum_{k = 1}^{K} \(\frac{\varphi}{1 - \psi \lambda_{k}} + c\) \langle\s, [\U]^{-1}_k\rangle [\U]_k.
\end{align}
}%
Notice that the \ARMA{1} retains the same graph frequency response as in the time-invariant case~\eqref{eq:ARMA1_response}, now expressed in the basis of $\LTV$. 
It is not difficult to show that the \ARMA{1} graph filter converges asymptotically. Let us denote the \rev{distance} between the filter output at two different time instants $t_1 > t_2$ as 
\begin{align}\label{eq:err_defin}
	\epsilon_{t_1,t_2} &= \frac{\norm{\z_{t_1} - \z_{t_2}}}{ x_\textit{max}}.
\end{align}
where $x_\textit{max} = \max_{t = 1, \ldots, t_1} \norm{\x_t}$ constitutes an upper bound on the energy of the input. We can now claim
\begin{theorem}\label{theo: err_conv}
Given the \ARMA{1} recursion \eqref{eq:ARMA_TVTV} \revv{and given that the graph Laplacians are uniformly bounded for every $t$ $\|\L_t\| \le \mrad$}, the \rev{distance} $\epsilon_{t_1,t_2}$  between the filter output at time instants $t_1$ and $t_2$ is upper-bounded as
\begin{align}\label{eq:err_fin_theo}
\begin{split}
	\epsilon_{t_1,t_2} &\leq \norm{\y_0} \frac{ |\psi \varrho|^{t_1} + |\psi \varrho|^{t_2} }{x_\textit{max} } + |\varphi| \frac{|\psi\varrho|^{t_2} - |\psi\varrho|^{t_1}}{1 - |\psi\varrho|} \\&\quad+ |c|~\frac{\| \x_{t_1-1} - \x_{t_2-1}\|}{x_\textit{max}}.
	\end{split}
\end{align}
\end{theorem}

\vskip2mm
(The proof is deferred to the appendix.)
\vskip2mm

\revv{For simplicity, set $c=0$ and consider $t_1$ big enough such that the term $|\psi \varrho|^{t_1} \approx 0$. 
Then, directly from \eqref{eq:err_fin_theo} we can find the value of $t_2$ such that the error between the two is smaller than a desired positive constant $\varepsilon$, \ie}
\begin{align}\label{eq.time}
t_2 \geq \log{(\alpha/\varepsilon)} \quad \Rightarrow \quad \epsilon_{t_1,t_2} \leq \varepsilon, 
\end{align}
with $\alpha = \norm{\y_0}/x_\textit{max} + \abs{\varphi} / (1 - \abs{\psi \varrho})$.
%
 
The results of Theorem \ref{theo: err_conv} can be extended to the general ARMA$_K$ graph filter. For the parallel implementation, we can proceed in the same way as for the ARMA$_1$ by considering that the output signal is the sum of $K$ ARMA$_1$ graph filters. Meanwhile, for the periodic implementation we can see that its form \eqref{eq:periodic_ext}, after one cyclic period, is analogous to \eqref{eq:ARMA1TV_G}. 

\revv{The main result of Theorem~\ref{theo: err_conv} stands in the fact that the ARMA output will not diverge as long as the graph Laplacians of each realization $\mathcal{G}_t$ has uniformly bounded spectral norm and from \eqref{eq.time} the distance decreases exponentially. Further, for $t$ big enough and if $\LTV$ enjoys an eigendecompositon the result in \eqref{eq:convTVTV} gives us insights where the ARMA output converges. Numerical results suggest that the obtained output is generally close to the designed frequency response of the ARMA filter.}
 \vspace{-2mm}


\section{Numerical Results}
\label{sec:numerical_results}
To illustrate our results we simulate two different case-studies: one with a fixed graph and a time-varying graph signal, and one where both the graph and graph signal are time-varying. In the latter case, the ARMA performance is also compared to the state-of-the-art FIR filters designed in a universal manner~\cite{Shuman2011}. 
With the first case-study, we aim to show how the proposed filters operate on graph signals that have spectral content in both graph and temporal frequency domains. Meanwhile, with the second the goal is to illustrate the ARMA performance when the underlying graph topology is not static anymore, but varies with time. \revv{For all our simulations, the ARMA filters, if not differently mentioned, are initialized to zero (\ie $\y_0=\mathbf{0}$ and $\y_0^{(k)} = \mathbf{0}$ for all $k$) and the filter design is performed in a universal setting.}

\subsection{Variations on the Graph Signal}
\label{subsec:sigTV}
%

In this subsection, we present simulation results for time-varying signals. 
We consider a 0.5-bandlimited graph signal $\u_t$ oscillating with a fixed temporal frequency $\pi/10$, meaning that 
\begin{align} 
\revv{\left\langle \u_{t}, \bphi_n \right\rangle = 
\left\{
\begin{array}{lr}
   e^{\jmath \pi t /10} & \text{if } \lambda_n < 0.5 \\
   0             & \text{otherwise},
\end{array}
\right.}
\end{align}
where $\lambda_n$ is the $n$-th eigenvalue of the normalized graph Laplacian and $t$ it the time index.
The signal is then corrupted with a second interfering signal $\v_t$, oscillating with a temporal frequency $9\pi/10$ \revvv{with graph spectrum defined in the following in two different ways.}. \rev{In addition, the signal at each node is corrupted with i.i.d. Gaussian noise $\n_t$, with zero mean and variance $\sigma^2 = 0.1$.}
We then attempt to recover $\u_t$ by filtering it with a parallel \ARMA{5} graph filter, effectively canceling the interference $\v_t$ and attenuating the out of band noise. 
\rev{The ARMA filter is designed only in the graph frequency domain based on the GFT of $\u_t$, \ie to approximate an ideal low-pass filter in the graph domain with cut-off frequency $\lambda_c = 0.5$. Regarding the temporal part, we exploit the property of the filter to preserve the same graph frequency response as the static case for low temporal oscillations, while attenuating the contribution of high temporal frequencies.}  
Our simulations were conducted using a random geometric graph $G$ composed of 100 nodes placed randomly in a square area, with any two nodes being connected if they are closer than 15$\%$ of the maximum distance in the area, \revv{with an average degree of 11.8.}%

Depending on whether the interference is correlated with the signal or not, we distinguish between two scenarios:

\vskip3mm
\emph{i) Correlated signal interference.} In this scenario, the interference is self-induced, \revvv{meaning that at a given instant $t$, $\v_t$ and $\u_t$ share the same graph spectrum,} but oscillating at a higher temporal frequency (due for instance to electronics problems).
To provide intuition, in Fig.~\ref{fig.sig_spec}.a, we show the graph spectral content of $\u_0$ and $\u_0 + \v_0 + \n_0$. We can see that once corrupted by noise and interference, the graph signal presents significant spectral content across the graph spectrum, thus loosing its bandlimited nature. Meanwhile, Fig.~\ref{fig.sig_spec}.b depicts the real part of the graph spectral content of the filter output after 100 iterations (\ie well after the initial state is forgotten). Even though the figure cannot capture the effect of dynamics (as it solely focuses on $t = 100$), it does show that all frequencies above $\lambda_c = 0.5$ have been attenuated and the interference contribution in the band is reduced. 

\begin{figure}[!t]
\centering
\includegraphics[width=0.95\columnwidth, trim= 0.4cm 0 0.45cm 0.5, clip=on]{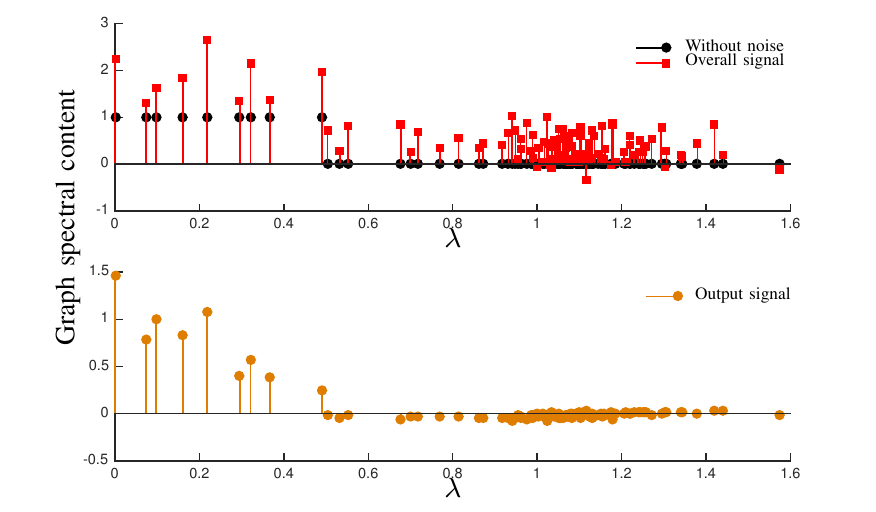}
\caption{Graph spectral content of the input signal as well as of the overall signal affected by interference and noise $a)$ (top), and of the filter output signal $b)$ (bottom). \revv{The output signal graph spectrum is shown for $t = 100$.}}
\vspace{-4mm}
\label{fig.sig_spec}
\end{figure}

\rev{To illustrate the filtering of the temporal frequencies of the signal, in Fig.~\ref{fig.sig_spec_time} we show the average spectrum over all nodes of the input and output signal. To increase visibility, the values in the figure are normalized with respect to the maximum. We can see that, the content relative to the interfering frequency $9\pi/10$ of the output signal is attenuated around 13 dB  with respect to the main temporal frequency content of $\pi/10$.}

\vskip3mm
\emph{ii) Uncorrelated signal interference.} Let us now consider a more involved scenario, in which the interfering graph signal satisfies
\begin{align}
	\revv{\langle \v_t, \bphi_n \rangle = e^{\jmath 9\pi t /10}e^{-\lambda_n},} 
\end{align}
\ie it is a signal having a heat kernel-like graph spectrum oscillating in time with a pulsation $\omega = 9\pi/10$. 
We will examine two types of errors: i) The first compares for each time $t$ the \ARMA{5} output GFT $\hat{\z}_t$ to that of the signal of interest 
\begin{align}
	{e}^{(\textit{total})}_t = \frac{\norm{\hat{\z}_t - \hat{\u}_t}}{\norm{\hat{\u}_t}}. 
\end{align}
Achieving a small error $e^{(\textit{total})}_t$ is a very challenging problem since an algorithm has to simultaneously overcome the addition of noise and the interference, while at the same time operating in a time-varying setting (see Fig.~\ref{fig.interf_err}).
ii) The second error focuses on interference and compares $\z_t$ to the output $\z^*_t$ of the same \ARMA{5} operating on $\u_t + \n_t$ (but not $\u_t + \v_t + \n_t$)
\begin{equation}
\label{eq:gr_err}
	\textnormal{e}^{(\textit{interf})}_t = \frac{\|\hat{\z}_{t} - \hat{\z}^*_{t} \|}{\| \hat{\z}^*_{t} \|},
\end{equation}
where $\hat{\z}^*_{t}$ is the GFT of $\z^*_{t}$.

We can see from Fig.~\ref{fig.interf_err} that after a few iterations this error becomes relatively small, which means that the output spectrum of the ARMA recursion when the signal is affected by interference is similar to when the interference-less signal is used. This gives a first insight, that using the ARMA recursion we can manage multiple signals on a graph by simply making them orthogonal in the temporal frequency domain. By a specific design of the filter coefficients, one can distributively operate on the graph signal of interest and ignore the others. Such a result cannot be achieved with FIR filters for two reasons: (i) they suffer from handling time-varying input signals, and (ii) the FIR filters do not operate on the temporal frequency content of the graph signals, thus such a distinction between overlapping signals is difficult to achieve.

The above results illustrate the conclusions of Section \ref{sec:time_variations}, and also quantify how much we can attenuate the signal at a specific graph/temporal frequency.
\vspace{-2mm}
\begin{figure}[!t]
\centering
\includegraphics[width=0.95\columnwidth, trim= 0.4cm 0 0.45cm 0.5, clip=on]{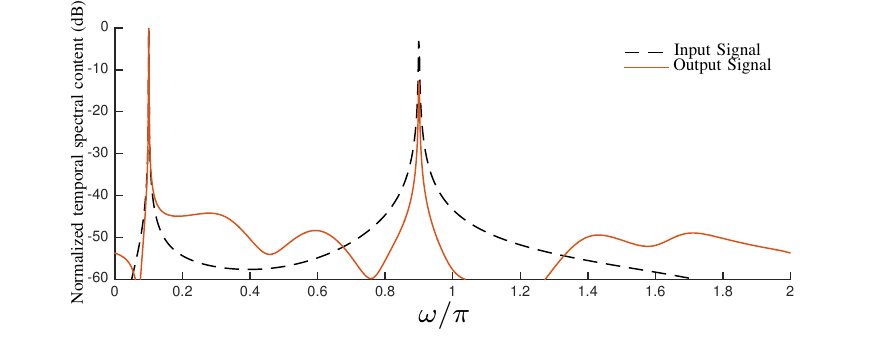}
\caption{Average time spectral content over all nodes of the input and output signal. The values are normalized with respect to the maximum.}
\label{fig.sig_spec_time}
\end{figure}

\begin{figure}[!t]
\centering
\includegraphics[width=0.95\columnwidth, trim= 0.4cm 0 0.45cm 0, clip=on]{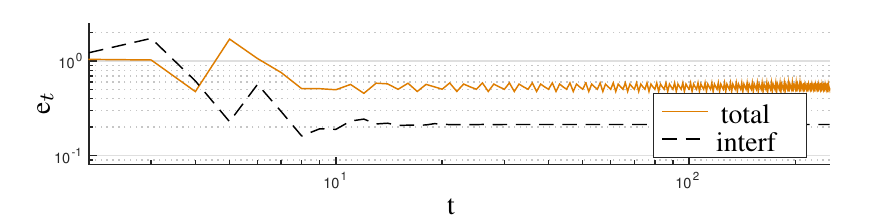}
\caption{Error of the ARMA recursion when the time-vaying input signal is affected by uncorrelated interference.}
\vspace{-4mm}
\label{fig.interf_err}
\end{figure}

\subsection{Variations on the Graph Topology}
\label{subsec:ARMA_FIR}
%
We examine the influence of graph variations for two filtering objectives. The first,  which corresponds to denoising, can be computed \emph{exactly} using ARMA. In the second objective, the graph filter is designed to \emph{approximate} an ideal low-pass graph filter, \ie a filter that eliminates all graph frequency components higher than some specific $\lambda_c$. 
In addition, we employ two different types of graph dynamics: \emph{random edge failures}, where the edges of a graph disappear at each iteration with a fixed probability, as well as the standard model of \emph{random waypoint} mobility~\cite{Bonnmotion2010}. 
The above setup allows us to test and compare universal ARMA and FIR graph filters \revv{(designed using the least-squares method)} over a range of scenarios, each having different characteristics.

\paragraph{Exact design (denoising).} We simulate the denoising problem (as defined by \eqref{eq:denoising}, with $w = 0.5$ and $K=1$) over the same graph topology of Section \ref{subsec:sigTV}, where the probability that an edge goes down at each time instant is $p = 0.05$.
The input signal $\x = \u + \n$ is given by a linear combination of a smooth signal $\u$ and noise $\n$. \rev{To ensure that the graph signal is smooth, \revv{we set its spectrum, \wrt the initial graph,} as $\langle \u, \bphi_n \rangle = e^{-5\lambda_n}$.} The noise $\n$ on the other hand is i.i.d. Gaussian distributed with zero mean and unit variance. 

\begin{figure}[!t]
\centering
\scriptsize
\includegraphics[width=1\columnwidth, trim= 0.3cm 0 0.75cm 0, clip=on]{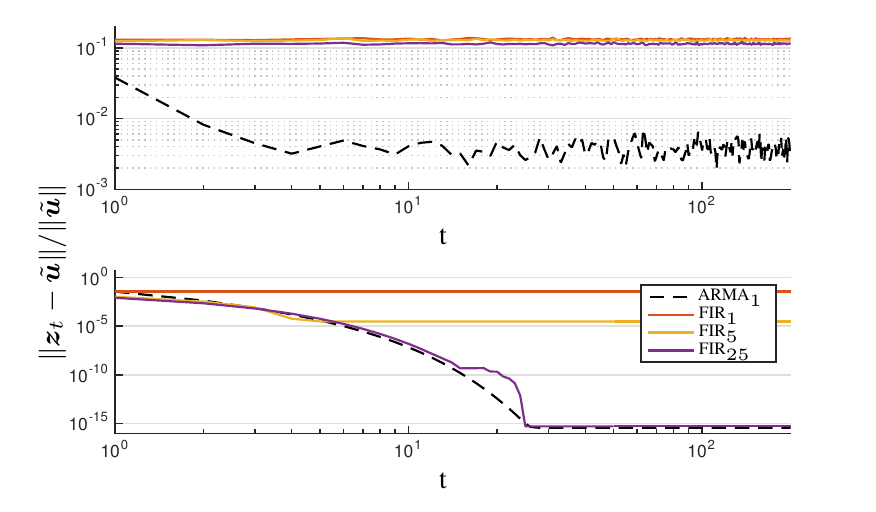}
\caption{Normalized error related to the solution of the denoising problem in a distributed way with graph filters. Results relative to random time-varying graph (top) and static graph (bottom). We compare the results of ARMA$_1$ with different FIR graph filters. \revv{The FIR$_K$ output at time $t$ is calculated as $\y_t = \sum_{k=0}^Kh_k\Phi_\L(t,t-k+1)\x$ and is not arrested after $K$ time instants.}}
\vspace{-4mm}
\label{fig.Den_err}
\end{figure}

To compare the results, we calculate the normalized error between the graph filter output and the analytical solution of the optimization problem~\eqref{eq:denoising} \rev{solved \wrt the initial graph.}
In Fig. \ref{fig.Den_err}, we plot the normalized error of solving the denoising problem via distributed graph filtering. We consider an ARMA$_1$ graph filter (designed according to Section~\ref{subsec:gra_den_interp} \revv{with $\y_0 = \x$}) and we compare its performance with FIR graph filters of different orders.
\revv{As expected, we can see that in the static case the ARMA graph after $K$ iterations has the same performance as the FIR$_K$ filter and thwy both match the solution of the optimization problem. On the other hand, in the random time-varying graph the ARMA filter outperforms all the FIRs. This is mainly due to its implementation strategy, which allows the ARMAs to handle the graph variations better. Also note that the result obtained from the ARMA$_1$ in the time-varying scenario quantifies the theoretical derivations in \eqref{eq:convTVTV} and Theorem~\ref{theo: err_conv}. Indeed, we can see that the obtained output is close (up to an order 10$^{-3}$) to the desired frequency response and the convergence is linear.}

We can see that, for both the random time-varying and static graph the ARMA graph filter gives a lower error with respect to the solution of the optimization problem. As we have seen before, for static graphs the ARMA filter matches correctly the analytical solution. Meanwhile, when the graph is generated randomly it approximates quite well the latter. On the other hand, the performance of the FIR filters is limited by the fact that they only approximate the solution of the optimization problem. \rev{Notice that the FIR output is given after $K$ time instants and then the filter is reset, hence the spikes in the figure.} 

\paragraph{Approximate design (ideal low pass).}
We use graph filters of increasing orders, specifically $K = 2, 4$ and $6$, \revv{to universally approximate} a low-pass graph filter with frequency response $g^*(\lambda) = 1$ if $\lambda < 0.5$, and zero otherwise. 
We consider a graph with 100 nodes living in a square of 1000 $\times$ 1000 meters, with a communication range of 180 meters. We simulated node mobility according to the random waypoint model~\cite{Bonnmotion2010} with a constant speed selected in $[0, 3]$ $m/s$.

We start with a scenario where only the graph topology changes in time whereas the graph signal remains invariant. Then, we simulate a more general case, where both the graph topology and the graph signal are time-varying. For both scenarios, we perform 20 distinct runs, each lasting 10 minutes and consisting of 600 iterations (one iteration per second). \rev{We then compare the response error $\|g - g^{*}\|/\|g^{*}\|$ of the ARMA filters with that of the analogous FIR filters while accounting for the initialization phase (we ignore the first 100 iterations). \revv{More specifically, at each time instant, we compute $g(\lambda_n) = \hat{y}_n/\hat{x}_n$, where the points $\hat{x}_n \approx 0$ are not considered. Then, it is compared with the desired frequency response at the particular graph frequency $\lambda_n$, \ie $g^*(\lambda_n)$.} The statistical significance of our results stems not only by the 20 distinct repetitions, but also by the large number of graph topologies experienced in each run.}

\begin{figure}[!t]
\centering
\begin{tabular}{c}
\hspace{-5mm}\includegraphics[width=1.05\columnwidth]{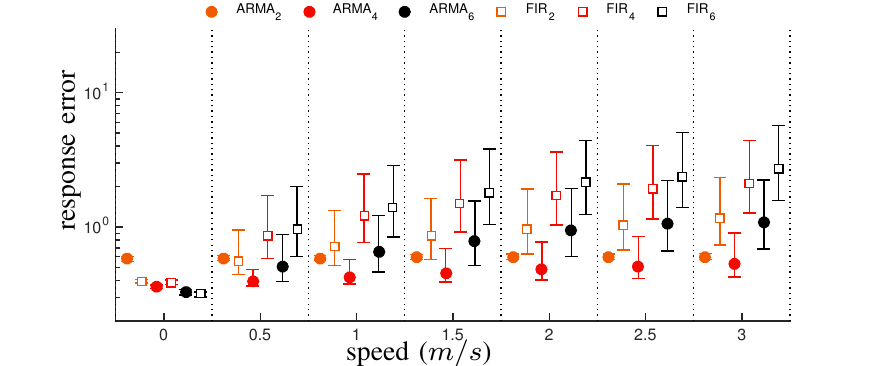}
\\ \\
\hspace{-5mm}\includegraphics[width=1.05\columnwidth]{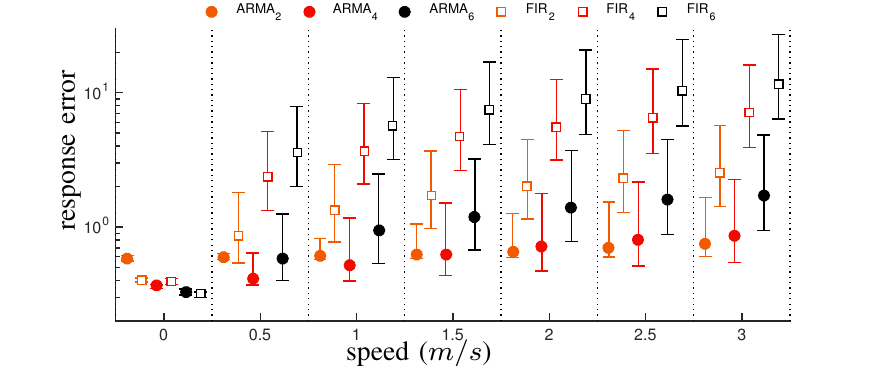}
\end{tabular}
\caption{The effects of the variations only on the graph topology (top) and on both graph and graph signal (bottom). The response error is calculated as $\|g(\lambda) - g^{*}(\lambda)\|/\|g^{*}(\lambda)\|$. Each error bar shows the standard deviation of the approximation error over 20 runs. A small horizontal offset is included to improve visibility.}
\vspace{-4mm}
\label{fig.TV_G}
\end{figure}

\vspace{3mm}
\emph{Time-varying graph, constant graph signal.} For this scenario, $\x$ is a random vector with entries selected uniformly distributed in [0, 1]. In Fig.~\ref{fig.TV_G} (top) we show the response error for increasingly higher node speeds. As expected, the error increases with speed. Nevertheless, the ARMA filters show a better performance in comparison to their analogous FIR filters. This indicates that the proposed approach handles better time-vaying settings than the FIR filters. Further, we can see that higher order ARMA filters approximate better the desired frequency response (smaller error) when the graph is static. On the other hand, when mobility is present, higher order ARMA recursions lead to a rough approximation due to their slower convergence and the fact that the poles go closer to the unit circle (larger coefficients).

\vspace{3mm}
\emph{Time-varying graph and graph signal.} To conclude, we simulate the more general case where both the graph structure and the graph signal change in time. Simulating a target tracking scenario, we let  the signal at each node take a value of zero, unless a node was within 100 meters from a target point, residing at the middle of the $1000\times~1000$ meter simulation area, in which case the node's value was set to one. In Fig.~\ref{fig.TV_G} (bottom) we show the response error as a function of the node's speed. It is not surprising that letting the graph signal change over time makes the graph filtering problem harder and the corresponding errors of all graph filters larger. 
\revv{As expected, the error increases with speed. Nevertheless, the ARMA filters show a better performance in comparison to their analogous FIR filters for all cases other than when $K=2$ and speed is zero (the latter is an artifact of the Shank's method).}


\section{Conclusions}
\label{sec:conclusions}
\setlength{\extrarowheight}{2.0pt}
\setlength{\tabcolsep}{9pt}
\begin{table*}[!t]
\linespread{1.5}
\centering
\caption{ Residues $r_k$ and poles $p_k$ of parallel ARMA$_K$ filter, for $K = 3, 5$ and $7$.
}
\label{table:coeff}
\begin{center}
\vspace{-4mm}
 \begin{tabular}{@{}c  @{}c @{}c @{}c @{} c  @{}c @{} c @{} c } \toprule
\textit{\ order\ } & $r_0, p_0$ &  $r_1, p_1$ & $r_2, p_2$ & $r_3, p_3$ & $r_4, p_4$ & $r_5, p_5$ & $r_6, p_6$ \\  
 \midrule 
K=3 & \hspace{3mm}\begin{tabular}{@{}c@{}} 10.954 + 0$i$, \\ -6.666 + 0$i$\end{tabular}  & \hspace{3mm}\begin{tabular}{@{}c@{}} 1.275 + 1.005$i$, \\ 0.202 + 1.398$i$\end{tabular}  & \hspace{3mm}\begin{tabular}{@{}c@{}} 1.275 - 1.005$i$, \\ 0.202 - 1.398$i$\end{tabular} & - & - & - &  -  \\ [0.0ex]
\rowcolor{black!7}[0pt][0pt]K=5 & \hspace{3mm}\begin{tabular}{@{}c@{}} -7.025 + 0$i$, \\ -3.674 + 0$i$\end{tabular}  & \hspace{3mm}\begin{tabular}{@{}c@{}} -1.884 - 1.298$i$, \\ -0.420 + 1.269$i$\end{tabular}  & \hspace{3mm}\begin{tabular}{@{}c@{}} -1.884 + 1.298$i$, \\  -0.420 - 1.269$i$\end{tabular} & \hspace{3mm}\begin{tabular}{@{}c@{}} 1.433 - 1.568$i$, \\ 0.703 + 1.129$i$\end{tabular}& \hspace{3mm}\begin{tabular}{@{}c@{}}  1.433 + 1.568$i$,\\ 0.703 + 1.129$i$\end{tabular}&  -  &  - \\ 
K=7 & \hspace{3mm}\begin{tabular}{@{}c@{}} -46.398 + 0$i$, \\ --3.842 + 0$i$\end{tabular}  & \hspace{3mm}\begin{tabular}{@{}c@{}} -20.207 - 8.343$i$, \\ 0.102 + 1.427$i$\end{tabular}  & \hspace{3mm}\begin{tabular}{@{}c@{}} -20.207 + 8.343$i$, \\  0.102 + 1.427$i$\end{tabular} & \hspace{3mm}\begin{tabular}{@{}c@{}} -5.205 + 4.946$i$, \\ -0.785 + 1.128$i$\end{tabular}& \hspace{3mm}\begin{tabular}{@{}c@{}} -5.205 - 4.946$i$,\\ -0.785 + 1.128$i$\end{tabular}& \hspace{3mm}\begin{tabular}{@{}c@{}}  3.124 - 10.622$i$, \\ 0.902 + 1.011$i$\end{tabular} & \hspace{3mm}\begin{tabular}{@{}c@{}} 3.124 + 10.622$i$, \\ 0.902 - 1.011$i$\end{tabular} \\ 
 \bottomrule
\end{tabular}
\end{center}
\vspace{-1mm}
\end{table*}

In this work, we presented the ARMA recursion as way of implementing IIR graph filters in a distributed way. We showed two different options to approximate any desired graph frequency response with an ARMA filter of order $K$, namely the parallel and periodic implementations. Experiments show that, our Shanks-based design method produces stable filter, which can approximate arbitrary well any desired graph frequency response. Furthermore, they attain linear convergence. 
\revv{The proposed ARMA graph filters were shown to provide solutions for two important graph filtering tasks: $(i)$ Tikhonov and Wiener graph denoising and $(ii)$ graph signal interpolation under smoothness assumptions.}

Characterized by a rational frequency response, ARMA graph filters can track time-varying input signals. In this case, we showed that our filters naturally extend to a  2-dimensional frequency space simultaneously operating in the graph- and time-frequency domain. In this way, we can distributedly filter a signal jointly in both domains, instead of operating on each of them separately, which leads to higher costs. \rev{Though we did not provide solutions for the joint design problem, we illustrated that, due to a connection between the poles in the graph domain and those in the Z-domain, graph filters which are designed only \wrt the graph frequency domain, are characterized by a specific temporal behavior.}
Further, we characterized the ARMA recursion when also the graph structure varies in time and proved that the linear convergence can be guaranteed also in this setting.

Our future research will be based on finding analytical stable design methods for both 1 and 2-dimensional ARMA recursions.
Furthermore, we are also interested to extend the proposed 2-dimensional graph filter to a separable case in order to obtain a disjoint filter design in each domain.


\appendix
\section{Deferred Proofs}
\label{app:dereffer_proofs}
\vspace{-2mm}
\subsection*{Table I}
For completeness and reproducibility, we include in Table~\ref{table:coeff} the filter coefficients of a parallel ARMA$_K$ filter approximating the step function with cut-off $\lambda_c = 0.5$ (\ie filter response equal to 1 for $\lambda < \lambda_c$ and zero otherwise) for $K = 3, 5, 7$. Higher order filters are omitted due to space considerations. 

\vspace{-2mm}
\subsection*{Proof of Theorem~\ref{theo:periodic_static}}

Define matrices $\bGamma_t = \theta_t {\bf I} + \psi_t \L$ and $\bPhi_\Gamma{(t,t')} = \bGamma_{t} \bGamma_{t-1} \cdots \bGamma_{t'}$ if $t \geq t'$, whereas $\bPhi_\Gamma{({t,t'})} = {\bf I}$ otherwise. 
%
%
The output at the end of each period can be re-written as a time-invariant system
\vspace{-2mm}
\begin{subequations}\label{eq:periodic_ext}
\begin{align}
  \y_{(i+1)K} &= \overbrace{\bPhi_\Gamma{(K-1,0)}}^{\triangleq \A} \y_{iK} + \overbrace{ \sum_{k = 0}^{K-1} \bPhi_\Gamma{(K-1,k+1)} \varphi_{k}}^{\triangleq \B} \x \\
  \z_{(i+1)K} &= \y_{(i+1)K} +  c \x.
\label{eq:recursion_coarse}
\end{align}
\end{subequations}
Both $\A$ and $\B$ have the same eigenvectors $\bphi_n$ as $\L$. \rev{Notice that \eqref{eq:periodic_ext} resembles \eqref{eq:ARMA1} and we can proceed in an identical manner}. As such, when the maximum eigenvalue of $\A$ is bounded by $\abs*{\lmax(\A)} < 1$, the steady state of~\eqref{eq:recursion_coarse} is 
\vspace{-1mm}
\begin{equation}
 \z = (I-\A)^{-1} \B\x + c \x = \sum_{n=1}^N \left(c+\frac{\lambda_n(\B)}{1 - \lambda_n(\A)} \right)\hat{x}_n \bphi_n.
\end{equation}
\revv{To derive the exact response, we exploit the backward product in the definition of $\bPhi_\Gamma{(t_1,t_2)}$ and we obtain} 
\begin{align}
	\lambda_n(\bPhi_\Gamma{(t_1,t_2)}) = \prod_{\tau = t_1}^{t_2} \lambda_n(\bGamma_t) =  \prod_{\tau = t_1}^{t_2} \left(\theta_{\tau} + \psi_{\tau}\lambda_n \right),
\end{align}
which, by the definition of $\A$ and $\B$, yields the desired frequency response. 
The linear convergence rate and stability condition follow from the linear convergence of~\eqref{eq:recursion_coarse} to $\y$ with rate $|\lmax(\A)|$.
%

\vspace{-2mm}
\subsection*{Proof of Theorem~\ref{theorem:ARMAKTV_parallel}}

The recursion of a parallel \ARMA{K} with time-varying input is 
\vspace{-2mm}
\begin{subequations}
\begin{empheq}{align}
  \y_{t+1}^{(k)} &= \psi^{(k)}\L \y_t^{(k)} + \varphi^{(k)}\x_t \ (\forall k) \\
  \z_{t+1} &= \sum_{k = 1}^{K} \y^{(k)}_{t+1} +  c \x_t, 
\end{empheq}
\end{subequations}

where $\y^{(k)}_t$ is the state of the $k$th \ARMA{1}, whereas $\x_t$  and $\z_t$ are the input and output graph signals, respectively. Using the Kronecker product the above takes the more compact form
\begin{subequations}
\begin{empheq}{align}
\y_{t+1} &= \left(\bPsi \otimes \L\right) \y_{t} + \bvarphi \otimes \x_t \\
\z_{t+1} &= (\mathbf{1}^\transp\otimes\I_N) \y_{t+1} + c \, \x_t,
\end{empheq}
\label{eq:ARMAK_TV_kron}
\end{subequations}
\rev{with $\y_t = \big[\y^{(1)\top}_t, \y^{(2)\top}_t, \cdots, \y^{(K)\top}_t \big]^\top$ the $NK \times 1$ stacked state vector,} $\bPsi = \text{diag}(\psi^{(1)}, \psi^{(2)}, \cdots, \psi^{(K)})$ a diagonal $K\times K$ coefficient matrix, $\bvarphi = (\varphi^{(1)}, \varphi^{(2)}, \cdots, \varphi^{(K)})^\top$  a $K\times 1 $ coefficient vector, and $\mathbf{1}$ the $K\times 1$ one-vector.
%
%
We therefore have
\vspace{-5mm}

\begin{align*}
\y_{t+1} &= \left(\bPsi \otimes \L\right)^t \y_0 + \sum_{\tau = 0}^t \left(\bPsi \otimes \L\right)^\tau \left( \bvarphi \otimes \x_{t - \tau} \right) \\
	     &= \left(\bPsi^t \otimes \L^t\right) \y_0 + \sum_{\tau = 0}^t \left(\bPsi^\tau \bvarphi \right) \otimes \left( \L^\tau \x_{t - \tau} \right).
\label{eq:eq1}
\end{align*}
Notice that, when the stability condition $ \norm*{\psi^{(k)} \L} < 1$ is met, $\lim_{t\rightarrow\infty}\norm*{\left(\bPsi^t \otimes \L^t\right) \y_0} = 0$. Hence, for sufficiently large $t$, the \ARMA{k} output is
\vspace{-2mm}
\begin{align*}
\lim_{t\rightarrow\infty} \z_{t+1} &= \lim_{t\rightarrow\infty}  \sum_{\tau = 0}^t (\mathbf{1}^\transp\otimes\I_N)\left(\bPsi^\tau \bvarphi \right) \otimes \left( \L^\tau \x_{t - \tau} \right)  + c \, \x_t\\
		&= \lim_{t\rightarrow\infty} \sum_{\tau = 0}^t \left(\mathbf{1}^\transp\bPsi^\tau \bvarphi \right) \otimes \left( \L^\tau \x_{t - \tau} \right) + c \, \x_t \\
		 &= \lim_{t\rightarrow\infty} \sum_{\tau = 0}^t \sum_{k= 1}^{K} \varphi^{(k)} \left(\psi^{(k)} \L \right)^\tau \x_{t - \tau} + c \, \x_t,
\end{align*}
where we have used the Kronecker product property $(\A\otimes\B)(\C\otimes\D) = (\A\C)\otimes(\B\D)$ and expressed the Kronecker product as the sum of $K$ terms.
The transfer matrix $\H(z)$ is obtained by taking the Z-transform in both sides and re-arranging the terms
\begin{align*}
\H(z) &= z^{-1} \sum_{k= 1}^{K} \varphi^{(k)} \sum_{\tau = 0}^\infty  \left(\psi^{(k)} \L \right)^\tau z^{- \tau} + c z^{-1}. 
\end{align*}
Finally, applying the GFT and using the properties of geometric series we obtain the joint transfer function in closed-form expression
\begin{align*}
  	H(z,\mu) &= z^{-1} \sum_{k= 1}^{K} \varphi^{(k)} \sum_{\tau = 0}^\infty  \left(\psi^{(k)} \lambda \right)^\tau z^{- \tau} + c z^{-1} \\
  	&= \sum_{k= 1}^{K} \frac{\varphi^{(k)} z^{-1} }{1- \psi^{(k)} \lambda z^{-1}} + c z^{-1}
\end{align*}
and our claim follows.
\vspace{-3mm}
\subsection*{Proof of Theorem~\ref{theorem:ARMAKTV_periodic}}

\rev{Recall for comodity $\bGamma_t = \theta_t {\bf I} + \psi_t \L$ and $\bPhi_\Gamma{(t,t')} = \bGamma_{t} \bGamma_{t-1} \cdots \bGamma_{t'}$ if $t \geq t'$, whereas $\bPhi_\Gamma{({t,t'})} = {\bf I}$ otherwise.}
\rev{Then, expanding recursion~\eqref{eq:ARMAK_periodic} for a time-varying input signal, we find that at the end of the $i$-th period, the filter output is $\z_{iK} = \y_{iK} + c \x_{iK-1} $, where 
\begin{align*}
	\y_{i+1K} &= {\bPhi_\Gamma{(iK-1,0)}} \y_{0}  + \sum_{k = 0}^{iK-1} \bPhi_\Gamma{(iK-1, k+1)} \varphi_{k} \x_k.
\end{align*}
For sufficiently large $i$ and assuming that the stability condition of Theorem~\ref{theorem:ARMAK_periodic} holds, the first term approaches the zero vector and can be ignored without any loss of generality.

We proceed by restricting the input graph signal to $\x_{iK} = x_{iK} \bphi$, where $\lambda, \bphi$ is an eigenpair of $\L$ (similarly $\y_{iK} = y_{iK} \bphi$ and $\z_{iK} = z_{iK} \bphi$). For compactness we introduce the shorthand notation $\lambda_k = \theta_k + \lambda \psi_k$ and $L = \prod_{\tau = 0}^{K-1} \lambda_\tau$. We then have 
\begin{align*}
	y_{iK} &= \sum_{k = 0}^{iK-1} \left(\prod_{\tau = k+1}^{iK-1} \lambda_\tau \right) \varphi_k x_k,
\end{align*}
which, after taking the Z-transform, becomes
\begin{align*}
	\frac{Y(z)}{X(z)} &= \sum_{k = 0}^{iK-1} \left(\prod_{\tau = k+1}^{iK-1} \lambda_\tau \right) \varphi_k z^{k-iK} \\
	&= \sum_{j = 0}^{i-1} L^{i-j-1} z^{(j-i)K}  \left( \sum_{k = 0}^{K-1} \left(\prod_{\tau = k+1}^{iK-1} \lambda_\tau \right) \varphi_k z^k \right) .
\end{align*}
The last step exploited the periodicity of coefficients in order to group the common terms of periods $j = 1, \ldots, i - 1$.
In the limit, the first term approaches
\begin{align*}
	\!\lim_{i\rightarrow\infty} \sum_{j = 0}^{i-1} L^{i-j-1} z^{(j-i)K} \!&=\! \lim_{i\rightarrow\infty} L^{-1} \sum_{j = 0}^{i-1} \left(\frac{L}{z^K}\right)^{i-j} \!\!= \frac{1}{z^K - L} \\
\end{align*}
Putting everything together, we find that the joint transfer function of the filter is
\begin{align*}
	H(z, \mu) = \frac{Z(z)}{X(z)} = \frac{\sum_{k = 0}^{K-1} \left(\prod_{\tau = k+1}^{K-1} \lambda_\tau \right) \varphi_k z^k}{z^{K} - L} + c z^{-1}
\end{align*}
and, after normalization, the claim \eqref{eq:ARMAKTV_periodic} follows.
}

\subsection*{Proof of Theorem~\ref{theo: err_conv}}
We start the proof by substituting the expression \eqref{eq:ARMA_TVTV} for $t_1$ and $t_2$ into the numerator of \eqref{eq:err_defin}. Then, we can write
\begin{align}
	\norm{\z_{t_1 + 1} - \z_{t_2+1}} &= \| \psi^{t_1+1} \bPhi_\L(t_1, 0) \y_{0} - \psi^{t_2+1} \bPhi_\L(t_2, 0) \y_{0}  \nonumber \\
	&\hspace{-12mm} + \varphi \sum_{\tau=0}^{t_1} \psi^\tau \bPhi_\L(t_1, t_1-\tau+1) \x_{t_1-\tau} + c \x_{t_1} \nonumber \\
	&\hspace{-12mm} - \varphi \sum_{\tau=0}^{t_2} \psi^\tau \bPhi_\L(t_2, t_2-\tau+1) \x_{t_2-\tau} - c \x_{t_2} \| .
\end{align}
Rearranging the terms, we have
\begin{align}
	\norm{\z_{t_1 + 1} - \z_{t_2+1}} & =\norm{ \psi^{t_1+1} \bPhi_\L(t_1, 0) \y_{0} - \psi^{t_2+1} \bPhi_\L(t_2, 0) \y_{0}\nonumber \\
	&\hspace{-16mm} + \varphi \sum_{\tau=t_2 + 1}^{t_1} \psi^\tau \bPhi_\L(t_1, t_1-\tau+1) \x_{t_1-\tau} + c( \x_{t_1} - \x_{t_2})} \nonumber 
\end{align}
%
By using the Cauchy-Schwarz property, the triangle inequality of the spectral norm, and a uniform bound $\varrho$ on the eigenvalues of matrices $\M_t$, the above expression simplifies  
\begin{align}
\begin{split}
	\norm{\z_{t_1 + 1} - \z_{t_2+1}} &\leq \left(\abs{\psi\varrho}^{t_1+1} + \abs{\psi \varrho}^{t_2+1} \right) \norm{\y_{0}}\\
	&\hspace{-15mm}\quad + \abs{\varphi} \sum_{\tau=t_2 + 1}^{t_1} \abs{\psi\varrho}^\tau \norm{\x_{t_1-\tau}} + \abs{c} \norm{\x_{t_1} - \x_{t_2}}.
	\end{split}
\end{align}
Leveraging the fact that $|\psi\varrho| < 1$, as well as that $ \norm{\x_t} \leq x_\textit{max} $ for every $t$, we can express the sum in a closed form 
\begin{align}
	\sum_{\tau=t_2 + 1}^{t_1} \abs{\psi\varrho}^\tau \norm{\x_{t_1-\tau}} \leq x_\textit{max} \left(\frac{\abs{\psi \varrho}^{t_2+1} - \abs{\psi \varrho}^{t_1+1}}{1 - \abs{\psi \varrho}}\right).
\end{align}
We obtain the desired bound on $\epsilon_{t_1,t_2}$ by dividing the above expressions with $x_\textit{max}$ and adjusting the \revv{indices}.
\bibliographystyle{IEEEtran}
\bibliography{bibliography}
\balance

\balance

\vspace{-0.4in}
\begin{IEEEbiography}[{\includegraphics[width=1in,height=1.25in,clip,keepaspectratio]{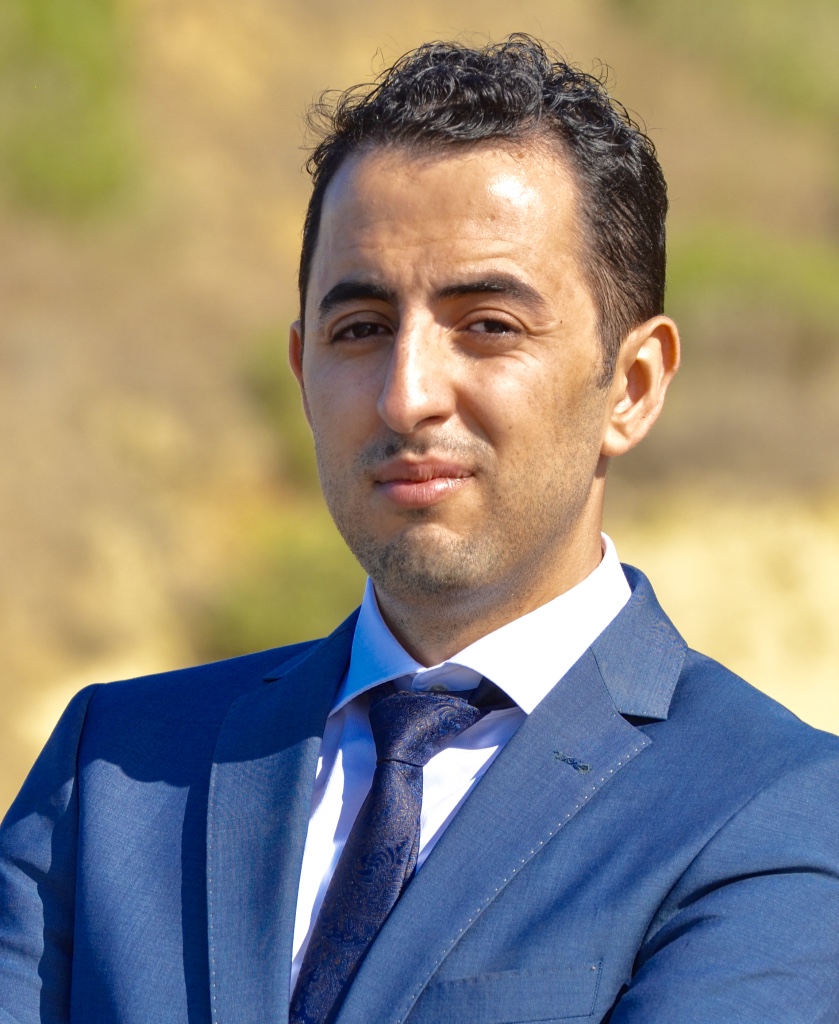}}]{Elvin Isufi}
was born in Albania in 1989. He received his Master of Science degree (cum laude) in Electronic and Telecommunication Engineering from University of Perugia, Italy, in 2014. From November 2013 to August 2014 he was a visiting member at Circuits and Systems group, Delft University of Technology, where he worked on his master thesis. Since November 2014 he is pursuing the Ph. D. degree on signal processing on graphs at Delft University of Technology. His research interests include signal processing on graphs, network coding and underwater communications.
\end{IEEEbiography}
\vspace{-0.4in}
\begin{IEEEbiography}[{\includegraphics[trim={10cm 0 15cm 0},clip,width=1in,height=1.25in]{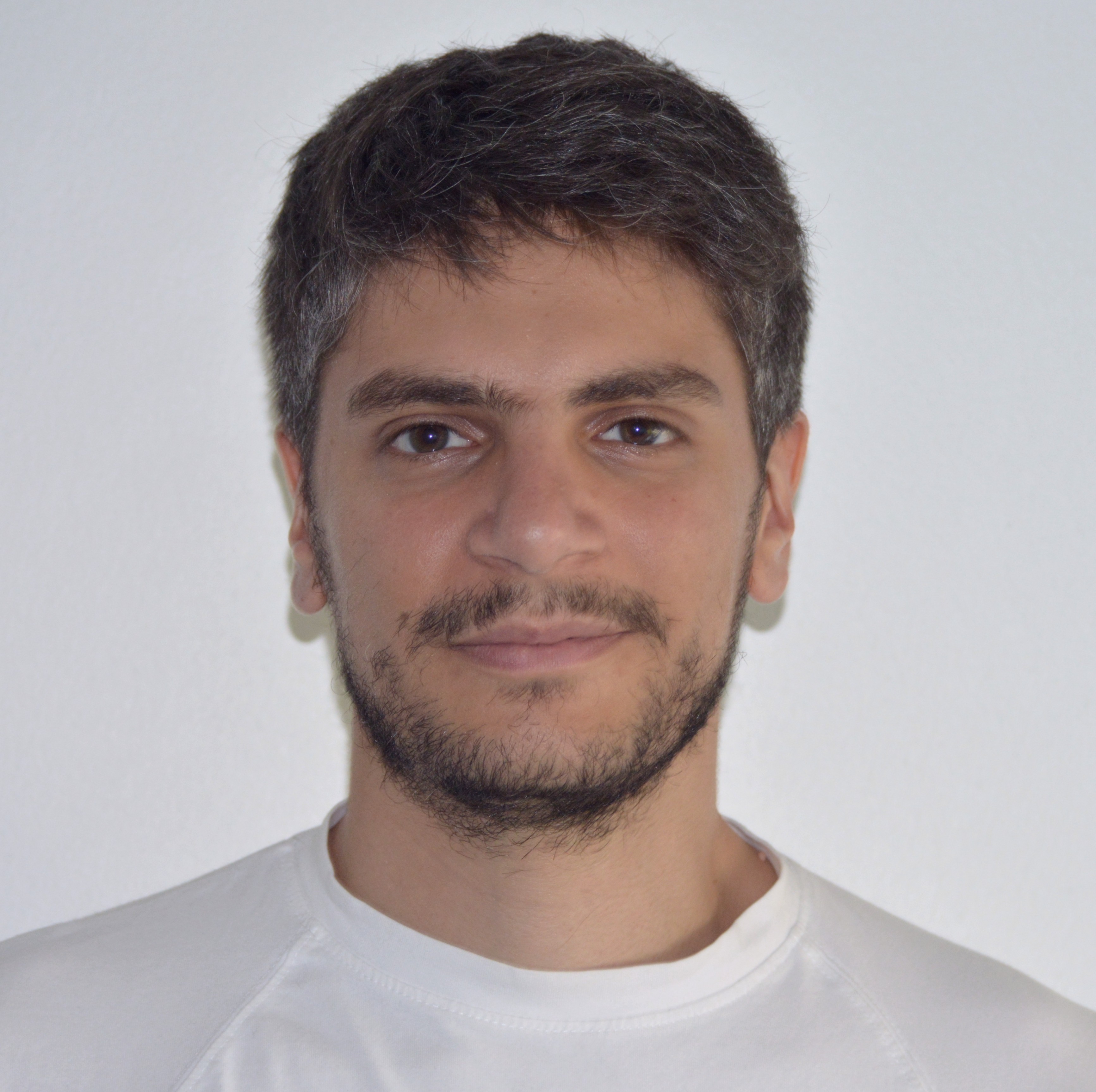}}]{Andreas Loukas}
is a research scientist jointly hosted by the LTS2 and LTS4 signal processing labs of the \'Ecole polytechnique f\'ed\'erale de Lausanne. His research interest lie in the intersection of data analysis, graph theory, and signal processing. Andreas holds a doctorate in Computer Science from Delft University of Technology, where he focused on distributed algorithms for information processing, and a Diploma in Computer Science from the University of Patras. 
\end{IEEEbiography}
\vspace{-0.4in}
\begin{IEEEbiography}[{\includegraphics[width=1in,height=1.25in,clip,keepaspectratio]{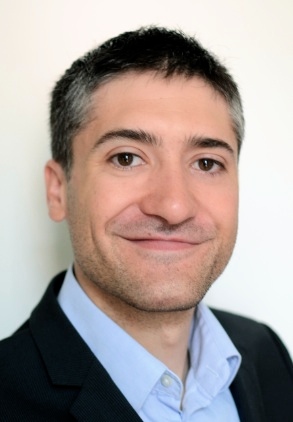}}]{Andrea Simonetto}
received the Ph.D. degree in systems and control from Delft University of Technology, Delft, The Netherlands, in 2012. He is currently a Post-Doctoral Researcher with the ICTEAM institute, at Universit\'e catholique de Louvain, Belgium. He was a Post-Doctoral Researcher with the Electrical Engineering Department, at Delft University of Technology. His current research interests include distributed estimation, control, and optimization.
\end{IEEEbiography}
\vspace{-0.4in}
\begin{IEEEbiography}[{\includegraphics[width=1in,height=1.25in,clip,keepaspectratio]{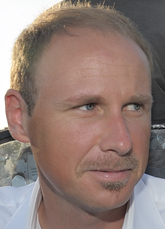}}]{Geert Leus}
received the MSc and PhD degree 
in applied sciences  from the Katholieke Universiteit Leuven, Belgium, 
in June 1996 and May 2000, respectively. Currently, Geert Leus is an 
``Antoni van Leeuwenhoek" Full Professor at the Faculty of Electrical 
Engineering, Mathematics and Computer Science of the Delft University 
of Technology, The Netherlands. His research interests are in the area 
of signal processing for communications. Geert Leus received a 2002 
IEEE Signal Processing Society Young Author Best Paper Award and a 2005 
IEEE Signal Processing Society Best Paper Award. He is a Fellow of the IEEE and a Fellow of EURASIP. Geert Leus was the Chair of the IEEE  Signal Processing for 
Communications and Networking Technical Committee, and an Associate 
Editor for the IEEE Transactions on Signal Processing, the IEEE 
Transactions on Wireless Communications, the IEEE Signal Processing 
Letters, and the EURASIP Journal on Advances in Signal Processing. 
Currently, he is a Member-at-Large to the Board of Governors of 
the IEEE Signal Processing Society and a member of the IEEE Sensor Array 
and Multichannel Technical Committee. He finally serves as the Editor 
in Chief of the EURASIP Journal on Advances in Signal Processing. 
\end{IEEEbiography}

\end{document}